\documentclass[twocolumn,10pt]{article}

\usepackage[letterpaper,margin=0.75in,columnsep=0.25in]{geometry}
\usepackage{amsmath,amssymb,amsthm}
\usepackage{bm}
\usepackage{booktabs}
\usepackage{multirow}
\usepackage{graphicx}
\usepackage{enumitem}
\usepackage{framed}
\usepackage{float}
\newfloat{algofloat}{t}{loa}
\floatname{algofloat}{Algorithm}
\newenvironment{algo}[1]{%
  \begin{algofloat}\caption{#1}\begin{framed}\small
  \begin{enumerate}[leftmargin=1.6em,itemsep=1.5pt,label=\arabic*:]}%
  {\end{enumerate}\end{framed}\end{algofloat}}
\usepackage{microtype}
\usepackage[colorlinks=true,linkcolor=blue,citecolor=blue,urlcolor=blue]{hyperref}

\newtheorem{proposition}{Proposition}
\newtheorem{definition}{Definition}
\newtheorem{remark}{Remark}
\usepackage{tikz}
\newcommand{\Otap}{\textsc{Otap}}      
\newcommand{\OtapS}{\textsc{OnTrack}}  

\newcommand{\Tm}{\mathbf{T}}          
\newcommand{\Cm}{\mathbf{C}}          
\newcommand{\Dm}{\mathbf{D}}          
\newcommand{\Gm}{\mathbf{G}}          
\newcommand{\Km}{\mathbf{K}}          
\newcommand{\mub}{\bm{\mu}}
\newcommand{\nub}{\bm{\nu}}
\newcommand{\ones}{\mathbf{1}}
\newcommand{\refset}{\mathcal{R}}
\newcommand{\actset}{\mathcal{A}}
\newcommand{\lookahead}{\mathcal{L}}
\newcommand{\cov}{\mathrm{cov}}
\newcommand{\leak}{\ell}
\newcommand{\mbar}{\bar{m}}

\title{\OtapS{}: Real-Time Monitoring and Intervention in LLM Agent
Trajectories via Streaming Structure-Aware Optimal Transport}

\author{%
  Babak Barazandeh\thanks{Corresponding author}, 
  Connor Swanson, 
  Chinmay Kulkarni, 
  Nikhil Mungel \\
  \small Cribl AI Research Lab \\
  \small \texttt{\{bbarazandeh, cswanson, ckulkarni, nmungel\}@cribl.io}
}
\date{}
\begin{document}
\maketitle


\begin{abstract}
    
Agents are deployed in a wide range of applications ranging from trip planners and stock trading to IT incident triage. In most of these applications, agents, by relying on the LLM, work autonomously with minimal rule-based safeguarding or no supervision, which results in cost and safety issues due to irreversible actions taken by agents. Recent works have tried to resolve these issues by either using another safeguard agent to monitor its behavior or doing a post-hoc evaluation of the logs. The first adds cost and latency to every step; the second delivers its verdict after the run, when the tokens are already burned and the damage done. 
To overcome this issue, we propose \OtapS{}, a streaming monitoring mechanism that, as the agent progresses in its task, compares its steps and the dependencies between them against recorded successful runs and alerts the user of potential issues or blocks the agent, taking about a millisecond per step. We study this problem in three regimes of decreasing access: full reference access (where we have access to historical runs of the task or similar variants, as well as tool schemas), the intermediate regime of accessing only tool schemas, and the final case of no prior knowledge—just the steps as they are being generated by the agent. Naturally, the expectation of the monitoring capability of \OtapS{} is reduced as data access is reduced, ranging from the detection of any violation of the plans to the detection of loops, stalls, and repeated tool calls without new results. Finally, we evaluate the performance of our proposed solution using SWE-bench trajectories. Evaluating based on the first 8 steps, the proposed method ranked the failing trajectories below succeeding ones better than content similarity-based approaches (+0.057 AUROC). Additionally, by deploying an abort policy on top of it, we save $\sim$18\% of the compute that would have been burned on agents heading to failure, and 83\% of runs that were interrupted were actually heading to failure (5 out of every 6 aborts were correct).

\end{abstract}

\section{Introduction}

Agents are being deployed across different applications and industries. One of their advantages is the autonomy that they bring during execution. With minimal human interference, they can execute a given task. They can automatically plan a trip, make a reservation, and complete the payment. However, this autonomy, if unchecked, can cause a damaging impact from cost and safety aspects.

On the cost side, these agents are using significant resources for a single step in a task, but they could easily suffer from hallucinations, out-of-order tool calls, or infinite loops—all of which are costly due to the underlying LLM usage. From a safety perspective, these agents—unmonitored—can delete or modify a record by mistake or, as in recent incidents, break out of their sandbox environment and start hacking other systems as reported in the OpenAI–HuggingFace incident~\cite{floridi2026emperor}. All these aspects require developing a monitoring system for agents.

The existing methods rely either on offline (after-the-fact) evaluations, where traces and logs are evaluated against a rule-based mechanism or language models are used as the judge~(i.e., LLM-as-a-Judge~\cite{zheng2023judging}). However, by the time an issue is found, the run has already been paid for. Online safeguards, on the other hand, monitor agents' execution using another agent. First, this method is costly and also adds latency to the execution due to an extra call to an LLM. On top of that, to reduce the cost, these safeguards are usually equipped with smaller models which struggle to detect the mistakes of the larger models, or a dedicated model is trained to score every step which renders the mechanism a black box~\cite{lightman2024verify}. Moreover, none of these methods look at the structure or the \emph{topology} of the runs, i.e., which steps depend on which, even though that is where loops or out-of-order calls show up.

To overcome this issue, we propose a streaming monitoring mechanism called \OtapS{}.
This approach models the agent's progress toward completing the task as a directed acyclic graph (DAG) in which each step of the graph corresponds to one node and, if a step's output is used in another step, there is an edge between them.
As the run unfolds, the monitor compares this growing graph against reference solutions using optimal transport and flags potential issues in about a millisecond per step. The references are optional, however—the monitor runs with whatever is available. With tool schemas alone, it still catches loops and stalls and blocks an irreversible action whose declared prerequisites are missing; with nothing but the event stream, loop and stall detection remain. Nothing is reconfigured across these regimes; parts of the monitor simply switch off as their inputs disappear.

Note that modeling agent task execution as a DAG and aligning it against reference solutions with optimal transport were already proposed in Otap~\cite{barazandeh2026otap}. However, their evaluation was a post hoc evaluation against complete references. Our contribution is making this alignment work on a live stream, which cannot be done by just re-running the batch method after each step due to the following structural challenges:
\begin{description}[leftmargin=1.2em,itemsep=2pt]

\item[S1 (prefix vs.\ whole).] Using a post hoc evaluation method in the middle of the trajectory will penalize the agent for every reference step it has not reached yet. So, the comparison punishes the agent for being early rather than being wrong—exactly the stage where the monitor most needs to stay quiet. Additionally, considering the portion of the reference that has been covered mid-stream is meaningless. For example, covering 10 of the 30 reference steps could be perfect progress if the agent is at step 10, and can be an indicator of a stuck agent if it is at step 80.

\item[S2 (provisional structure).] Edges are discovered later when a step (node) uses an output of another step. As a result, a new step immediately after just being generated can seem disconnected and might look like the agent is hallucinating. Judging new steps on their structure immediately can flood the system with lots of false positives. Note that extracting dependencies in post hoc evaluation is a dominant source of error~\cite{barazandeh2026otap}, and it even gets harder in streaming when the monitor cannot look ahead. 

\textbf{Note}: \textbf{S1} is concerned about the steps not taken by the agent when comparing with reference, and \textbf{S2} is concerned about the steps that the agent has taken but are not clear yet. As a simple example, the agent does a search at step 5, which at the moment the step looks disconnected. Ten steps later, the agent edits that searched file. That's when the edge appears and the first step becomes clear.
\item[S3 (compute budget).] The evaluator must produce a verdict before the
agent's next tool call. Re-solving the full alignment from scratch at every
step is unnecessary and, as the trajectory grows, quickly becomes
infeasible.
\item[S4 (decisions, not scores).] Post hoc evaluation methods can output a score, and the operator can define a threshold to evaluate if the execution was successful or not. The streaming monitoring mechanism can't. It instead outputs an action: \textit{continue, warn}, or \textit{halt}. To better understand this, consider an agent at step 5 that takes one hallucinated step. This is one fifth of all steps, so this will result in the aggregated score going under the threshold, while if this happened at step 30, it might move the score slightly. As a result, the same score value means different things at different lengths. So, the score is informative, but it should not be used as a trigger and no fixed threshold can drive the decision. A streaming monitoring mechanism should focus on per-step signals whose meaning stays the same as the trajectory grows.
\end{description}

\OtapS{} addresses all four challenges. For S1, the monitor only expects
the parts of the reference the agent could have reached so far, so an
agent is never punished for being early (Sec.~\ref{sec:frontier}). For
S2, a new step's structural influence starts at zero and phases in as
later steps use its output; edges discovered late are inserted into the
graph in place (Sec.~\ref{sec:structure}). For S3, instead of re-solving
the alignment at every step, we take one warm-started iteration of the
batch solver per event --- about a millisecond per step --- and run the
solver to convergence only before acting on an intervention
(Sec.~\ref{sec:cg}). For S4, decisions come from per-step signals with a
short grace window for new steps, never from the aggregate score
(Sec.~\ref{sec:verdicts}). At the end of the run, one final solve makes
the monitor's output identical to the batch evaluator's --- provided the
dependencies discovered online match what batch extraction would find
(Prop.~\ref{prop:anytime}). 
\section{Related Work}

\paragraph{Trajectory evaluation.} Evaluating an agent trajectory still means choosing between a cheap score and a detailed one. Outcome-only benchmarks~\cite{mialon2024gaia,zhou2024webarena,liu2024agentbench} collapse a run to pass or fail, so the steps that produced the outcome never enter the score. Text metrics that match a trace against a reference (exact match, BLEU/ROUGE, embedding alignment~\cite{zhang2020bertscore}) do look inside the run, but they assume the task has one correct wording. A broken plan that resembles the reference can then rank above a valid plan written differently, and~\cite{barazandeh2026otap} shows that this ranking can fall below chance. LLM-as-a-judge~\cite{zheng2023judging} and process reward models~\cite{lightman2024verify} score individual steps, so they are not stuck on a single reference. The cost is practical: a judge adds latency, cost, and run-to-run noise, and a reward model needs step-level annotation before it can be trained. These methods also wait until the run has finished. Observability platforms that sit on live traffic store detailed traces, but the checks they apply while a run is open are schema validation, token heuristics, and judges applied asynchronously to spans that have already closed. We are not aware of a deployed system that reasons about the structure of the execution graph while that graph is still growing.

\paragraph{Optimal transport.} The solver is built from standard parts. Entropic regularization and Sinkhorn iterations~\cite{cuturi2013sinkhorn} made optimal transport usable in practice. Unbalanced variants~\cite{chizat2018unbalanced} relax exact conservation of mass, so mass that should not be matched can be discarded. Gromov--Wasserstein~\cite{memoli2011gw} compares relational structure, and its fused form~\cite{vayer2019fgw} does the same for attributed graphs. Warm-started incremental Sinkhorn schemes already exist for problems that change slowly~\cite{peyre2019computational}. \OtapS{} uses that warm start. What changes between solves is more specific: the support grows by one row with each event, and the reference marginal moves as the agent progresses. That combination, as far as we know, is new.

\paragraph{Transport-based trajectory scoring.} The closest prior work is \Otap{}~\cite{barazandeh2026otap}, which scores a completed trajectory after the fact by aligning attributed execution DAGs. We adopt its graph formalism and its node costs. Three findings from its ablations constrained the streaming design. The structure term is load-bearing. Restricting the method to a single reference costs 16.9 AUROC points. And the largest source of error is heuristic dependency extraction. From there, this paper asks a different question. Frontier masking, the streaming solver, provisional-structure handling, the verdict engine, and the calibration finding have no counterpart in the batch setting, because a batch scorer only has to say how good a finished run was. \OtapS{} has to decide whether the run should be allowed to continue.

\section{Problem Formulation}
\label{sec:problem}

To formulate the agent behavior, consider how an agent tackles a problem. It generates a stream of steps in which, at each step it describes what it is doing, calls a tool with the required arguments, and receives its result. Steps here are not independent from one another; the artifact (a queried database, a tool result) produced at this step can be consumed as input to other steps down the road. As a result, this execution is not a flat sequence; it is a growing dependency graph. Importantly, these links are not always visible at the moment a step is taken. The importance of one step may only become clear when a later step uses its output. The following definition formalizes this trajectory.

\begin{definition}[Streamed trajectory prefix]
We define the agent's execution procedure as a stream of events $s_1, s_2, \dots$, where each event is a tuple consisting of $s_t = (a_t, \tau_t, c_t, e_t, \zeta_t)$. Here, $a_t$ is the action the agent is taking, $\tau_t$ is the tool identifier, $c_t$ is the arguments it uses, $e_t$ is the event's generated artifact, and $\zeta_t$ is the dependency metadata including parent indices and consumed/generated artifact identifiers, which both might be incomplete depending on the harness. The prefix graph $P_t = (V_t, E_t, \phi)$ has one node per event and contains all dependency edges known at time $t$, with $E_t$ potentially gaining new edges between existing nodes at a later time (retroactive discovery), and $\phi$ assigns each node the embeddings of its text fields.
\end{definition}
Note that the metadata at time $t$, $\zeta_t$, is harness-related and depends on its logs. Consider an example in which an agent is editing a config file at step 15 that was found during a search at step 5. The  event itself is the same under any harness:

\begin{quote}\footnotesize\begin{verbatim}
event 15:
  action: "edit the DB config to fix the timeout"
  tool:   edit
  args:   {file: "src/config.py", ...}
  result: "file saved"
\end{verbatim}\end{quote}

\noindent but $\zeta_{15}$ comes in three grades, depending on the harness:

\begin{description}[leftmargin=1.2em, itemsep=3pt, topsep=3pt]
\item[Full metadata.] The harness declares dependencies outright
(\texttt{parents:\,[5]}, plus consumed/produced artifact identifiers);
the edge $(5,15)$ is read off directly.

\item[Partial metadata.] No parents are declared, but the tool wrapper
logs the objects it touched (\texttt{consumed:\,["src/config.py"]}).
The monitor maintains a \emph{ledger} mapping each artifact identifier
to the step that first surfaced it, so a lookup recovers step 5 and the
same edge is inferred.

\item[No metadata.] The common case in practice, and
the case of our evaluation corpus. The monitor manufactures $\zeta_t$
itself: tool schemas or built-in rules specify where each tool's
artifacts appear in its arguments --- here, the \texttt{edit} schema
marks its \texttt{file} argument as consumed, yielding
\texttt{src/config.py} --- and the ledger lookup then proceeds exactly
as in the partial-metadata case.
\end{description}

\noindent This extraction is deterministic string matching, not an LLM
call --- cheap and auditable --- but it fails when a dependency leaves
no matchable identifier. Suppose the agent instead ran
\texttt{sed -i 's/timeout=5/timeout=30/' \$(find . -name '*.py' | head -1)}
while describing it as ``fix the timeout setting we saw earlier'': the
target file is resolved inside the shell at execution time and the
reference to step 5 lives only in the agent's context, so no harness,
however faithful, records a traceable link. This is why $\zeta_t$ is
\emph{possibly incomplete} by definition, and why
Sec.~\ref{sec:structure} treats the graph as provisional.

Besides the agent's own stream, the monitor may be given
\emph{references}: known-good executions of the same task --- recorded
runs that solved it correctly, whether curated by hand or mined from past
successes --- represented in the same form as the stream. References are
optional (see the regimes below). When they exist, they form
a reference set $\refset = \{R_1,\dots,R_K\}$, where reference $R_k$ has
$m_k$ steps. For each reference, we precompute a
structure matrix $\Dm^{R_k}$ that records, for every pair of its steps,
how far apart they sit in the dependency graph: the number of dependency
hops from one to the other, normalized to $[0,1]$, with unconnected
pairs maximally distant. This is the same construction used for post-hoc
scoring in~\cite{barazandeh2026otap}; we default to unsigned distances
(edge direction ignored), a choice revisited in
Sec.~\ref{sec:structure}. The prefix graph carries an analogous matrix
$\Dm^{P}$, grown row by row as events arrive. Figure~\ref{fig:setup} summarizes the setup.

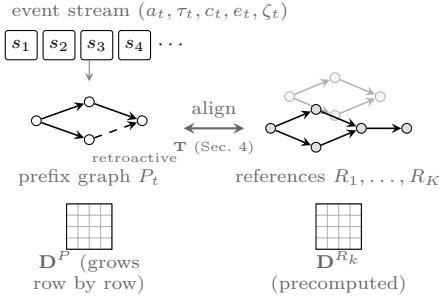
\begin{figure}[t]
\centering
\begin{tikzpicture}[font=\scriptsize, >=stealth,
  ev/.style={draw, rounded corners=1pt, minimum width=12pt, minimum height=11pt, inner sep=1pt},
  pn/.style={draw, circle, inner sep=1.3pt, fill=white},
  rn/.style={draw, circle, inner sep=1.3pt, fill=black!12},
  dep/.style={->, semithick}]
\node[ev] (s1) at (0,2.1) {$s_1$};
\node[ev] (s2) at (0.5,2.1) {$s_2$};
\node[ev] (s3) at (1.0,2.1) {$s_3$};
\node[ev] (s4) at (1.5,2.1) {$s_4$};
\node at (2.0,2.1) {$\cdots$};
\node[anchor=west, text=black!60] at (-0.25,2.55) {event stream $(a_t,\tau_t,c_t,e_t,\zeta_t)$};
\node[pn] (p1) at (0.2,1.1) {};
\node[pn] (p2) at (0.9,1.35) {};
\node[pn] (p3) at (0.9,0.85) {};
\node[pn] (p4) at (1.6,1.1) {};
\draw[dep] (p1) -- (p2);
\draw[dep] (p1) -- (p3);
\draw[dep] (p2) -- (p4);
\draw[dep, dashed] (p3) -- (p4);
\node[text=black!60] at (1.5,0.62) {\tiny retroactive};
\node[text=black!60] at (0.9,0.35) {prefix graph $P_t$};
\draw[->, black!50] (0.9,1.9) -- (0.9,1.62);
\draw[step=0.15, black!30] (0.6,-0.6) grid +(0.6,0.6);
\draw (0.6,-0.6) rectangle +(0.6,0.6);
\node[text=black!60, align=center] at (0.9,-0.9) {$\Dm^{P}$ (grows\\ row by row)};
\draw[<->, thick, black!60] (2.15,1.0) -- (2.95,1.0);
\node[text=black!60] at (2.55,1.25) {align};
\node[text=black!60] at (2.55,0.75) {\tiny $\Tm$ (Sec.~4)};
\begin{scope}[shift={(3.55,1.42)}, opacity=0.3]
\node[rn] (q1) at (0,0) {}; \node[rn] (q2) at (0.6,0.25) {};
\node[rn] (q3) at (0.6,-0.25) {}; \node[rn] (q4) at (1.2,0) {};
\draw[dep] (q1)--(q2); \draw[dep] (q1)--(q3); \draw[dep] (q2)--(q4); \draw[dep] (q3)--(q4);
\end{scope}
\begin{scope}[shift={(3.3,1.0)}]
\node[rn] (r1) at (0,0) {}; \node[rn] (r2) at (0.6,0.25) {};
\node[rn] (r3) at (0.6,-0.25) {}; \node[rn] (r4) at (1.2,0) {};
\node[rn] (r5) at (1.8,0) {};
\draw[dep] (r1)--(r2); \draw[dep] (r1)--(r3); \draw[dep] (r2)--(r4);
\draw[dep] (r3)--(r4); \draw[dep] (r4)--(r5);
\end{scope}
\node[text=black!60] at (4.2,0.35) {references $R_1,\dots,R_K$};
\draw[step=0.15, black!30] (3.9,-0.6) grid +(0.6,0.6);
\draw (3.9,-0.6) rectangle +(0.6,0.6);
\node[text=black!60, align=center] at (4.2,-0.9) {$\Dm^{R_k}$\\ (precomputed)};
\end{tikzpicture}
\caption{The monitor's inputs. Events arrive one at a time and grow the
prefix graph $P_t$ (dashed: a dependency discovered retroactively).
References $R_1,\dots,R_K$ are known-good executions in the same form.
Each graph carries a matrix of pairwise dependency-hop distances
($\Dm^{P}$ grows row by row; $\Dm^{R_k}$ is precomputed), and the
alignment $\Tm$ of Sec.~4 matches the two sides step by step.}
\label{fig:setup}
\end{figure}

\paragraph{Information regimes.}
Orthogonal to the harness's log quality --- which determines how well
the prefix graph can be \emph{built} --- is how much prior knowledge
the monitor is given to compare that graph \emph{against}.
We consider three nested settings, ordered by how much the monitor is
given, and its capabilities degrade explicitly across them. With
\emph{references and tool schemas}, everything below is available:
structural alignment against the references, per-step leakage, and the
pre-execution gate for irreversible actions. With \emph{schemas but no
references}, the monitor can still catch self-evident pathologies ---
loops, stalls, repeated calls that produce nothing new --- and still
block an irreversible action whose declared prerequisites are missing;
what it loses is any notion of deviation from a plan, since there is no
plan to deviate from. With \emph{neither}, only the self-referential
vital signs remain --- the honest minimum, and still enough to flag the
runaway-loop failure mode that dominates real cost overruns
(Sec.~\ref{sec:anomaly}). A fourth axis is reference \emph{quality}:
a hand-curated gold run, a mined past success, and a machine-generated
sketch do not deserve equal trust, so the authority of a verdict is
scaled by reference provenance (Sec.~\ref{sec:regimes}). Figure~\ref{fig:regimes} summarizes both axes.

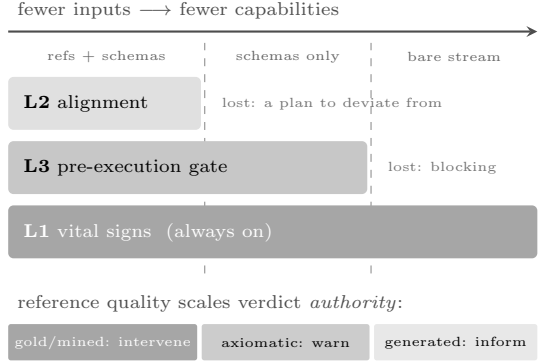
\begin{figure}[t]
\centering
\begin{tikzpicture}[font=\scriptsize, >=stealth]
\draw[->, thick, black!70] (0,3.05) -- (7.0,3.05);
\node[black!70, anchor=south west] at (0,3.1) {fewer inputs $\longrightarrow$ fewer capabilities};
\draw[dashed, black!45] (2.6,-0.15) -- (2.6,2.9);
\draw[dashed, black!45] (4.8,-0.15) -- (4.8,2.9);
\node[black!60] at (1.3,2.72) {\tiny refs $+$ schemas};
\node[black!60] at (3.7,2.72) {\tiny schemas only};
\node[black!60] at (5.9,2.72) {\tiny bare stream};
\fill[black!12, rounded corners=2pt] (0,1.75) rectangle (2.55,2.45);
\node[anchor=west] at (0.08,2.10) {\textbf{L2} alignment};
\node[black!55, anchor=west] at (2.70,2.10) {\tiny lost: a plan to deviate from};
\fill[black!22, rounded corners=2pt] (0,0.90) rectangle (4.75,1.60);
\node[anchor=west] at (0.08,1.25) {\textbf{L3} pre-execution gate};
\node[black!55, anchor=west] at (4.90,1.25) {\tiny lost: blocking};
\fill[black!35, rounded corners=2pt] (0,0.05) rectangle (7.0,0.75);
\node[anchor=west, text=white] at (0.08,0.40) {\textbf{L1} vital signs \ (always on)};
\node[black!70, anchor=west] at (0,-0.55) {reference quality scales verdict \emph{authority}:};
\fill[black!35, rounded corners=1pt] (0,-1.3) rectangle (2.5,-0.82);
\fill[black!20, rounded corners=1pt] (2.56,-1.3) rectangle (4.78,-0.82);
\fill[black!9, rounded corners=1pt]  (4.84,-1.3) rectangle (7.0,-0.82);
\node[text=white] at (1.25,-1.06) {\tiny gold/mined: intervene};
\node at (3.67,-1.06) {\tiny axiomatic: warn};
\node at (5.92,-1.06) {\tiny generated: inform};
\end{tikzpicture}
\caption{Information regimes. Each capability band runs for as long as
its inputs exist and terminates when they are removed: without
references there is no plan to deviate from (L2 ends); without schemas
there is nothing to gate (L3 ends); the vital signs of L1 never turn
off. Bands are stacked by how long they survive (bedrock at the
bottom), not by layer number; L1--L3 are the monitor's three layers,
defined in the architecture paragraph in Sec.~\ref{sec:4}. The strip underneath shows
the orthogonal quality axis: the richer the provenance of the
references, the more authority a verdict carries
(Sec.~\ref{sec:regimes}).}
\label{fig:regimes}
\end{figure}

\paragraph{Notation.}
Throughout the paper, $\phi$ embeds text into unit vectors, $d_{\cos}$
is cosine distance, and $d_{\mathrm{tool}}(\tau, \tau')$ is a
tool-substitution cost: zero for the same tool, growing toward one as a
substitution becomes less plausible. For a coupling $\Tm$, the row sums
$\Tm\ones$ give the mass each agent step placed and the column sums
$\Tm^{\!\top}\ones$ the mass each reference step received, where
$\ones$ is the all-ones vector. The generalized divergence of
nonnegative vectors is
$\mathrm{KL}(a\,\|\,b) = \sum_i a_i \log(a_i/b_i) - a_i + b_i$ (with
$0\log 0 = 0$; annihilating a marginal entry $b_i$ costs exactly
$b_i$), and the negative entropy of a nonnegative matrix is
$H(\Tm) = \sum_{ij} T_{ij}(\log T_{ij} - 1)$. Per-reference losses are
aggregated by a soft-min with temperature $\tau_{\mathrm{sm}} = 0.05$
(distinct from the tool identifier $\tau_t$). Finally,
$\langle \cdot, \cdot \rangle$ denotes the Frobenius inner product,
$[x]_{+} = \max(x, 0)$, and $[x]_{[0,1]}$ clips to the unit interval.

\paragraph{The streaming problem.}
At every event $t$, output (i) a diagnostic state (running score, per-node
alignment, coverage), and (ii) a verdict
$v_t \in \{\textsf{OK}, \textsf{EXPLORING}, \textsf{WARN}, \textsf{LOOP},$
$\textsf{STALLED}, \textsf{CAUSAL\_INVERSION}, \textsf{OFF\_TRACK},
\textsf{BLOCK}\}$, subject to a per-event compute budget far below tool-call
latency, such that at termination the evaluator produces the same output
as a batch evaluator running the identical implemented procedure~(Sec.~\ref{sec:cg}) on the
completed trajectory.
\section{The \OtapS{} Evaluator}
\label{sec:4}
This section builds the monitor in two parts. We first recall the
batch alignment that \OtapS{} inherits --- the cost table and the
objective --- and the monitor's three-layer architecture; we then
develop the streaming components that answer the four challenges of
the introduction: the mass semantics (Sec.~\ref{sec:mass}), the
frontier mask for S1 (Sec.~\ref{sec:frontier}), the streaming solver
for S3 (Sec.~\ref{sec:cg}), provisional-structure handling for S2
(Sec.~\ref{sec:structure}), and the verdict engine for S4
(Sec.~\ref{sec:verdicts}).
\paragraph{Ingredients: the cost table and the alignment.}
Two matrices are the ingredients~\cite{barazandeh2026otap}. The first is a
\emph{cost table} $\Cm$: entry $C_{ij}$ prices how implausible it is
that agent step $i$ plays the role of reference step $j$, combining the
similarity of their action descriptions, of their arguments, and of the
tools they call:
\begin{align}
C_{ij} ={}& \alpha\, d_{\cos}(\phi(a_i), \phi(a^R_j))
+ \beta\, d_{\cos}(\phi(c_i), \phi(c^R_j)) \notag\\
&+ \delta\, d_{\mathrm{tool}}(\tau_i, \tau^R_j),
\label{eq:nodecost}
\end{align}
where the weights $\alpha, \beta, \delta$ trade off the three
ingredients --- we use $(0.45, 0.20, 0.35)$ throughout, a setting that
worked well across all our experiments and was never retuned; every
component is normalized to $[0,1]$.
The second is the \emph{alignment} itself, a
nonnegative matrix $\Tm$: entry $T_{ij}$ says how much of agent step $i$
is matched to reference step $j$ --- a soft, fractional correspondence
rather than a hard one-to-one assignment. Each agent step carries mass
$\mu_i$ and each reference step mass $\nu_j$; their values are a design
decision with real consequences, settled in Sec.~\ref{sec:mass}. Figure~\ref{fig:toyexample}
illustrates both on a toy example: valid steps place their mass on
their reference counterparts, while an anomalous step's row stays
empty --- and that emptiness is the detection signal.

\begin{figure}[t]
\centering
\begin{tikzpicture}[font=\scriptsize, >=stealth, x=0.46cm, y=0.46cm]
\foreach \x/\l in {0/search,1/read,2/edit,3/test}{
  \node[rotate=35, anchor=south west, black!70] at (\x+0.1,3.05) {\tiny \l};
  \node[rotate=35, anchor=south west, black!70] at (\x+6.6,3.05) {\tiny \l};}
\node[anchor=east, black!70] at (-0.15,2.5) {\tiny grep error};
\node[anchor=east, black!70] at (-0.15,1.5) {\tiny edit migration};
\node[anchor=east, black!70] at (-0.15,0.5) {\tiny fetch weather (!)};
\foreach \y/\vals in {2/{5,60,80,90}, 1/{80,55,10,70}, 0/{95,90,92,97}}{
  \foreach [count=\x from 0] \v in \vals{
    \fill[black!\v] (\x,\y) rectangle +(1,1);}}
\draw[black!60, step=1] (0,0) grid (4,3);
\node[black!80, anchor=north] at (2,-0.3) {cost $\Cm$};
\draw[->, thick, black!60] (4.4,1.5) -- (6.1,1.5);
\node[black!60] at (5.25,2.1) {minimize};
\node[black!60] at (5.25,0.95) {$\mathcal{F}(\Tm)$};
\foreach \y/\vals in {2/{85,8,0,0}, 1/{0,8,85,0}, 0/{3,2,3,2}}{
  \foreach [count=\x from 0] \v in \vals{
    \fill[black!\v] (\x+6.5,\y) rectangle +(1,1);}}
\draw[black!60, step=1, shift={(6.5,0)}] (0,0) grid (4,3);
\node[black!80, anchor=north] at (8.5,-0.3) {alignment $\Tm$};
\node[black!70, anchor=north] at (8.5,-0.95)
  {\tiny row nearly empty: leakage $\leak_3 \approx 1$};
\draw[->, black!70] (10.45,-0.85) to[bend right=15] (10.7,0.3);
\end{tikzpicture}
\caption{The two ingredients on a toy example: three agent steps against a
four-step reference. In the cost table $\Cm$, dark means implausible:
``grep error'' matches \emph{search} cheaply, ``edit migration'' matches
\emph{edit}, and the hallucinated ``fetch weather'' is expensive against
every reference step. Minimizing Eq.~\eqref{eq:ufgw} yields the
alignment $\Tm$ (dark = matched mass): the two valid steps place their
mass on their counterparts, while the hallucinated step's row stays
nearly empty --- the objective preferred paying the unmatched-mass
penalty over forcing a bad match. That empty row \emph{is} the leakage
signal of Eq.~\eqref{eq:leak}.}
\label{fig:toyexample}
\end{figure}
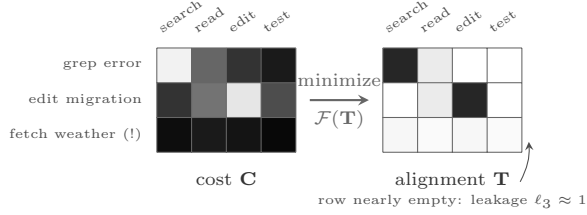
\paragraph{Objective.}
The alignment is chosen to balance four pressures: match cheaply (first
term below); respect structure, so that steps far apart in the agent's
graph map to steps far apart in the reference (second term); remain free
to leave an anomalous step unmatched, paying per unit of unmatched mass
(the two KL terms); and stay slightly diffuse, which is what makes fast
solvers possible (the entropy term). Formally, the cost of a (prefix)
graph against one reference is the minimum over $\Tm \ge 0$ of
\begin{align}
\mathcal{F}(\Tm) ={}& (1-\theta)\,\langle \Cm, \Tm\rangle
+ \frac{\theta}{2} \!\sum_{i,k,j,l}\! \big(D^{P}_{ik} - D^{R}_{jl}\big)^{2}
T_{ij}\,T_{kl} \notag\\
&+ \lambda\,\mathrm{KL}(\Tm\ones\,\|\,\mub)
+ \lambda\,\mathrm{KL}(\Tm^{\!\top}\ones\,\|\,\nub)
+ \varepsilon\,H(\Tm),
\label{eq:ufgw}
\end{align}
with structure weight $\theta = 0.35$ and $\varepsilon = 0.05$
throughout. The factor $\tfrac{1}{2}$ makes $\theta\,\Gm$ of
Eq.~\eqref{eq:gwgrad} the gradient of the quadratic term for symmetric
$\Dm$. The reported score soft-mins the per-reference diagnostic losses
$\ell_k$, $\tilde\ell = -\tau_{\mathrm{sm}} \log \tfrac{1}{K}\sum_k
e^{-\ell_k/\tau_{\mathrm{sm}}}$, and reports
$s = \exp(-[\tilde\ell\,]_{+})$; the floor at zero is what places
$s \in (0,1]$, since the entropic term can make raw losses negative.

\paragraph{Architecture of the monitor.}
\OtapS{} separates signals by their prerequisites into three decoupled
layers: \textbf{L1} \emph{vital signs} (reference-free intrinsic diagnostics:
loop detection, information gain, stall) run always; \textbf{L2}
\emph{transport alignment} (the streaming UFGW machinery, Secs.~\ref{sec:mass}--\ref{sec:verdicts})
requires references; \textbf{L3} \emph{pre-execution policy gate}
(tool-contract prerequisite checks plus a one-row proposal alignment,
Sec.~\ref{sec:verdicts}) requires schemas and is the only
\emph{synchronous} component --- it evaluates the proposed action before
the tool executes, exclusively for irreversible calls, so worst-case
added latency is confined to the rare actions that warrant it. L1 is not
optimal transport and we do not present it as such; it is the honest
degradation mode when no reference of any provenance exists.

\subsection{Mass semantics on prefixes}
\label{sec:mass}

Batch evaluation uses $\mu_i = 1/n$, $\nu_j = 1/m$: both graphs are complete
and both totals equal one. On a prefix this breaks: with $\mu_i = 1/t$ the
partial trajectory carries total mass $1$ against a frontier-masked $\nub$
whose total is a \emph{fraction} of one (Sec.~\ref{sec:frontier}), so even a
perfectly matched step leaks mass by construction. \OtapS{} instead assigns
\textbf{constant per-node agent mass}
\begin{equation}
\mu_i \;=\; 1/\mbar, \qquad \mbar = \tfrac{1}{K}\textstyle\sum_k m_k,
\label{eq:mass}
\end{equation}
so one agent step carries the same nominal mass as an average reference
step (exact for references of size $\mbar$), the agent
total grows as $t/\mbar$, and the per-node normalized leakage
\begin{equation}
\leak_i \;=\; \Big[\, 1 - \frac{1}{\mu_i}\sum_j T_{ij} \,\Big]_{[0,1]}
\label{eq:leak}
\end{equation}
is a scale-free, calibrated trigger at every $t$; the clip to $[0,1]$ is
part of the definition, since in unbalanced OT a row mass can exceed
$\mu_i$ (over-transported rows are treated as fully matched). Granularity differences
remain the soft coupling's job exactly as in the batch setting. At
termination the escalated solve may renormalize to $1/n$ for exact batch
comparability.

\subsection{The prefix problem: frontier-masked marginals}
\label{sec:frontier}

Define reference-node coverage from the previous event's coupling (the
one-step lag resolves the circular dependency between mask and coupling):
\begin{equation}
\cov_t(j) \;=\; \frac{[\Tm_{t-1}^{\!\top}\ones]_j}{\nu_j},
\qquad
\text{$j$ \emph{satisfied}} \iff \cov_t(j) \ge \kappa .
\end{equation}
The \emph{active} set and its lookahead are
\begin{align}
\actset_t &= \{\, j : \text{every reference-parent of } j \text{ is
satisfied} \,\},\notag\\
\lookahead_t &= \mathrm{children}(\actset_t)\setminus\actset_t,
\end{align}
with roots vacuously active, and the masked marginal is
\begin{equation}
\nu^{(t)}_j =
\begin{cases}
\nu_j & j \in \actset_t \cup \{\text{satisfied}\},\\
\beta_{\mathrm{look}}\,\nu_j & j \in \lookahead_t,\\
\varepsilon_\nu\,\nu_j & \text{otherwise},
\end{cases}
\label{eq:mask}
\end{equation}
with defaults $\kappa{=}0.3$, $\beta_{\mathrm{look}}{=}0.3$,
$\varepsilon_\nu{=}0.01$. Three consequences:

\emph{(a) Order-freedom where the DAG permits it.} All members of the
currently executable antichain of $R$ are active simultaneously, so an agent
that executes independent reference steps in any topological order is never
penalized for order (Prop.~\ref{prop:perm}).

\emph{(b) Precision and recall acquire stream semantics.} Agent-side leakage
(Eq.~\ref{eq:leak}) is the \emph{real-time} signal --- it means hallucination,
digression, or prematurity \emph{now}. Reference-side coverage
$\rho_t = \sum_j \min(\cov_t(j),1)\,\nu_j \big/ \sum_j \nu_j$ is a progress
meter; full-trajectory recall is only computed at termination with the mask
removed.

\emph{(c) Prematurity is detectable and disambiguable.} A leaking step may
still match an \emph{inactive} (future) reference node under the unmasked
cost. The verdict engine performs this secondary check and distinguishes:
\textsf{hallucinated} (no match anywhere), \textsf{premature-valid} (matches
a future node whose producers are satisfied --- a legitimate jump beyond the
frontier), and \textsf{causal inversion} (matches a future node whose
producers are \emph{absent}) --- the streaming analogue of the causal-inversion perturbation studied
in~\cite{barazandeh2026otap}, detected \emph{at emission time}, before the
structure term could even observe the violation.

If the frontier stalls incorrectly (a covered step whose embedding similarity
failed to reach $\kappa$), everything downstream appears premature; a
persistent gap between masked and unmasked loss raises an internal
\textsf{FRONTIER\_STALL} flag that adaptively widens $\beta_{\mathrm{look}}$
before any user-facing verdict escalates.

\subsection{Keeping the quadratic term: streaming conditional gradient}
\label{sec:cg}

The batch ablations of~\cite{barazandeh2026otap} are unambiguous: at
$\theta=0$ the alignment degenerates into an unbalanced Wasserstein with
below-chance inversion detection. A streaming variant that replaces the
quadratic term with a row-wise linear proxy (comparing $\Dm^P_{i,:}$ to
$\Dm^R_{j,:}$ index-by-index) silently conditions on the identity alignment
and re-imports the exact-order assumption the alignment exists to remove.
\OtapS{} keeps the true objective by observing that the standard batch UFGW
solver --- conditional gradient (CG) --- is already an iterative
linearization scheme, and amortizing it:
\emph{one warm-started CG iteration per event}; the coupling tracks the
moving optimum the way stochastic gradient descent tracks a drifting loss.

Per event $t$ and reference $k$ (superscripts dropped):

\textbf{(1) Extend.} Append row $t$ to the node-cost matrix $\Cm$ of
Eq.~\eqref{eq:nodecost} (with cached reference embeddings) and
row/column $t$ to $\Dm^P$; pad the previous coupling
$\tilde\Tm = [\Tm_{t-1};\, \mathbf{0}_{1\times m}]$. Under the constant
per-node mass of Eq.~\eqref{eq:mass} no rescaling is required.

\textbf{(2) Linearize} the quadratic term at $\tilde\Tm$ via the square-loss
decomposition~\cite{peyre2019computational}: with
$p = \tilde\Tm\ones$, $q = \tilde\Tm^{\!\top}\ones$,
\begin{equation}
\Gm \;=\; (\Dm^{P})^{\circ 2} p\,\ones_m^{\!\top}
\;+\; \ones_t\, q^{\!\top} \big((\Dm^{R})^{\circ 2}\big)^{\!\top}
\;-\; 2\, \Dm^{P} \tilde\Tm\, (\Dm^{R})^{\!\top},
\label{eq:gwgrad}
\end{equation}
computed in $O(t^2 m + t m^2)$. For symmetric distance matrices,
$\theta\,\Gm$ is the gradient of the $\tfrac{\theta}{2}$-weighted
quadratic term of Eq.~\eqref{eq:ufgw}, so the linearized total cost
$\mathbf{M} = (1-\theta)\Cm + \theta\,\Gm$ is exactly the local
first-order model of the objective; the implementation additionally
clips negative entries of $\Gm$ to zero, keeping the linearized cost ---
and hence the Gibbs kernel --- bounded. Structural age weights
(Sec.~\ref{sec:structure}) scale row $i$ of $\Gm$ by $\omega_i(t)$.

\textbf{(3) Solve} the linearized unbalanced entropic OT with the total
cost $\mathbf{M}$ from step (2), Gibbs kernel
$\Km = \exp(-\mathbf{M}/\varepsilon)$, masked marginal $\nub^{(t)}$, and
warm-started duals $u \leftarrow [u_{t-1};\, \mathrm{med}(u_{t-1})]$,
$v \leftarrow v_{t-1}$, iterating with exponent
$\eta = \lambda/(\lambda+\varepsilon)$:
\begin{equation}
u \leftarrow \Big(\frac{\mub}{\Km v}\Big)^{\eta},
\qquad
v \leftarrow \Big(\frac{\nub^{(t)}}{\Km^{\!\top} u}\Big)^{\eta},
\label{eq:sinkhorn}
\end{equation}
for $n_{\mathrm{inner}}$ iterations (default 5, tolerance-early-stopped);
$\Tm_{\mathrm{new}} = \mathrm{diag}(u)\,\Km\,\mathrm{diag}(v)$.

\textbf{(4) Average:} $\Tm \leftarrow (1-\gamma)\tilde\Tm +
\gamma\,\Tm_{\mathrm{new}}$ with fixed $\gamma = 0.7$ (a closed-form FGW line
search is available but unnecessary at this scale).

\textbf{(5) Two-tier escalation.} The fast path above runs on every event.
If and only if it raises a would-be intervention verdict, the evaluator
escalates to a converged solve ($\le$10 outer CG, $\le$50 inner Sinkhorn,
cold tolerance) on the current prefix before the verdict is confirmed.
\emph{No intervention is ever taken on a stale linearization}, and the
converged-solve cost is paid exactly when it is cheapest relative to the tool
call about to be prevented.

\emph{Implementation notes: three distinct quantities.} The paper
handles three related but distinct objects, and we keep them separate
throughout. (i) The \emph{exact objective}, Eq.~\eqref{eq:ufgw}, defines
the alignment problem and is what the theory refers to. (ii) The
\emph{fast online optimizer} deviates from it in two documented ways:
the per-event linearization clips negative entries of the GW gradient to
zero (keeping the Gibbs kernel bounded), and the \emph{diagnostic} loss
reported alongside verdicts approximates the two KL penalties by an
$L_1$ total-mass surrogate
$\lambda\,(|\ones^{\!\top}\Tm\ones - \ones^{\!\top}\mub| +
|\ones^{\!\top}\Tm^{\!\top}\ones - \ones^{\!\top}\nub^{(t)}|)$ and clips
the quadratic form at zero. (iii) The \emph{finalization path} runs the
same implemented solver (clipped gradient, surrogate diagnostic) with
the mask removed and a converged schedule --- it is \emph{not} an exact
minimization of Eq.~\eqref{eq:ufgw}, and Prop.~\ref{prop:anytime} is
stated accordingly, as equality with a batch evaluator running the same
implementation. All verdict thresholds are calibrated directly on the
implemented quantities, so the deviations affect interpretation of the
loss scale, not the calibrated behavior.

\subsection{Provisional structure}
\label{sec:structure}

\textbf{Retroactive edges.} When event $k$ consumes an artifact produced at
step $i$, the evaluator inserts $(i,k)$, repairs ancestor sets by a
breadth-first pass over descendants of $k$ (bitset representation,
$O(t^2/64)$ words worst case), and recomputes the affected rows of $\Dm^P$.
The coupling is \emph{not} recomputed retroactively; the next CG step absorbs
the change --- a further payoff of the streaming-CG framing.

\textbf{Age weights.} A node's structural row is untrustworthy while its
artifacts are young. We set
\begin{equation}
\omega_i(t) = \min\!\big(1,\ (t - t_i)/h_{\mathrm{age}}\big),
\qquad h_{\mathrm{age}} = 3,
\label{eq:age}
\end{equation}
clamped to $\omega_i = 1$ the moment any artifact of $i$ is consumed (its
role is then evidenced, not assumed). Semantic and tool costs, which need no
lookahead, apply immediately. The finalization path disables age weights
entirely --- batch evaluation uses the unweighted structure --- so the
final score cannot depend on emission ages.

\textbf{Signed vs.\ unsigned.} In the batch setting, signed and unsigned
structural distances were found indistinguishable under artifact
tracing~\cite{barazandeh2026otap}, because tracing deletes an unsatisfiable
edge rather than reversing it. Streaming inherits this; we default to
unsigned magnitudes and note that causal
inversions are caught earlier and more directly by the premature-step check
of Sec.~\ref{sec:frontier}(c).

\subsection{Signals, lifecycle, and verdicts}
\label{sec:verdicts}

All triggers derive from per-event, per-node quantities
(Table~\ref{tab:signals}); the aggregate score is reported but never
thresholded. A batch result of \cite{barazandeh2026otap} --- benign
redundant leaves inflate the loss by $O(r/n)$ with an explicit
KL-destruction constant --- becomes the streaming \emph{tolerance band}
separating cheap abandoned exploration from genuine off-track mass.

\begin{table}[t]
\centering\small
\begin{tabular}{@{}p{2.1cm}p{5.6cm}@{}}
\toprule
Signal & Definition / role \\
\midrule
node leakage $\leak_t$ & Eq.~\eqref{eq:leak} on the newest node; hallucination
/ digression / prematurity, routed through disambiguation \\
match cost $\bar c_t$ & $\sum_j T_{tj} M_{tj} / (\sum_j T_{tj} +
\epsilon_{\mathrm{den}})$; ``matched, but expensively'' (wrong tool on
the right step); the $\epsilon_{\mathrm{den}}$ floor covers
fully-destroyed rows (distinct from the entropic $\varepsilon$) \\
coverage velocity & $\rho_t - \rho_{t-w}$; stall = activity without progress \\
loop score (L1) & near-duplicate predecessor (action, tool, arguments), with
a \textbf{retry exemption}: duplicates of an errored step are exempt up to
$r_{\max}$ attempts \\
info gain (L1) & $\Delta I_t = 1 - \max_{i<t}\cos(\phi(e_t), \phi(e_i))$ over
effect embeddings; spinning = new calls, no new state \\
strategy switch & the soft-min $\arg\min_k$ changes; benign once, oscillation
warns \\
\bottomrule
\end{tabular}
\caption{Per-event signals. All are $O(t)$ or free by-products of the
transport update.}
\label{tab:signals}
\end{table}

\textbf{Exploration lifecycle.} A newly emitted node with
$\leak_t > \theta_{\mathrm{leak}}$ enters \textsf{EXPLORING} with a grace TTL
of $g$ events (default $3$). The flag clears if a later step consumes its
artifact (branch validated) or a newly activated reference node absorbs it.
On TTL expiry the unmasked secondary check classifies it: matches a future
reference node with satisfied producers $\Rightarrow$ premature-valid, clear;
producers absent $\Rightarrow$ \textsf{CAUSAL\_INVERSION}; matches nothing
$\Rightarrow$ unresolved orphan. When cumulative unresolved-exploration mass
exceeds a budget $B$ (default $25\%$), the verdict escalates to
\textsf{OFF\_TRACK}. Dead-end branches that were abandoned after one or two
steps land in \textsf{ABANDONED}: bounded cost, no alarm, by the tolerance
band above.

\textbf{L3 pre-execution policy gate.} Tools are classified reversible /
irreversible from schema annotations. The gate runs on the \emph{proposed}
action $q_t = (a_t, \tau_t, c_t)$ --- available before the tool executes,
before any post-state artifact exists, and before row $t$ enters the
coupling. It clears an irreversible proposal only if (a) every artifact
the schema requires exists in the live graph with a producing node, and,
when references are present, (b) a one-row \emph{proposal alignment}
succeeds: the nearest reference node under the pre-execution attribute
cost $\alpha\, d_{\cos}(\phi(a_t), \phi(a^R_j)) + \delta\,
d_{\mathrm{tool}}(\tau_t, \tau^R_j)$, evaluated against the frontier
\emph{recomputed from the current coupling} $\Tm_{t-1}$ inside the gate,
is active or satisfied and its cost is below
$\theta_{\mathrm{match}}$; otherwise \textsf{BLOCK}. (Recomputing rather
than reading the cached frontier matters: the cache is one event stale
for a pre-execution check, and an earlier implementation that used it
blocked valid calls --- see the gate tests in
Sec.~\ref{sec:controlled}.) No transport mass is
required --- the gate is a schema-plus-proposal check, which is what makes
synchronous operation possible. It is a policy gate, not a soundness
guarantee: its coverage is bounded by the schema annotations. Reversible
actions are never blocked, preserving cheap exploration.

\subsection{Multi-reference streaming and reference regimes}
\label{sec:regimes}

Independent coupling states are maintained per reference and aggregated by
temperature-controlled soft-min~\cite{barazandeh2026otap}; $K$ is small
(3--10) and cost is
linear in $K$. The running $\arg\min$ doubles as a live ``which strategy is
the agent following'' diagnostic. Reference provenance scales verdict
authority: \emph{gold} and \emph{mined} references (clustered successful
production traces) enable interventions; \emph{axiomatic} references built at
runtime from tool produces/consumes contracts enable L3 and downgraded
warnings; \emph{planner-generated} skeletons --- which re-import the
LLM-as-a-judge circularity --- may inform but never trigger an automatic halt.
With no reference at all, L2 is off and \OtapS{} degrades, explicitly, to
L1 + L3.

\begin{algo}{\OtapS{} per-event fast path}
\label{alg:otaps}
\item[] \textbf{Input:} event $s_t$; state
$\{\Tm^{(k)}, u^{(k)}, v^{(k)}\}_{k\le K}$, $\Dm^P$, lifecycle
\item \textbf{if} proposed action $q_t$ irreversible \textbf{and} L3 gate
fails (schema + proposal alignment against the $t{-}1$ frontier)
\textbf{return} \textsf{BLOCK}
\item admit $s_t$: parents from declared deps $\cup$ artifact consumption;
update ancestors, $\Dm^P$, age weights $\omega$
\hfill (Sec.~\ref{sec:structure})
\item \textbf{for} $k = 1 \dots K$:
\item \quad extend $\Cm^{(k)}$ by row $t$; update frontier mask
$\nub^{(k,t)}$ from $\Tm^{(k)}_{t-1}$ \hfill (Eq.~\ref{eq:mask})
\item \quad $\Gm \gets$ GW gradient at padded $\Tm^{(k)}$, rows scaled by
$\omega$ \hfill (Eq.~\ref{eq:gwgrad})
\item \quad $\Tm_{\mathrm{new}} \gets$ warm-started unbalanced Sinkhorn on
$(1{-}\theta)\Cm^{(k)} + \theta\Gm$ \hfill (Eq.~\ref{eq:sinkhorn})
\item \quad $\Tm^{(k)} \gets (1{-}\gamma)\Tm^{(k)}_{\mathrm{pad}} +
\gamma\,\Tm_{\mathrm{new}}$
\item compute signals (Table~\ref{tab:signals}) on best reference; run
lifecycle + verdict engine \hfill (Sec.~\ref{sec:verdicts})
\item \textbf{if} verdict $\in$ interventions: escalate --- converged solve,
recompute signals, re-verdict
\item \textbf{return} score (soft-min), coverage, leakage, verdict
\end{algo}

\section{Theoretical Properties}
\label{sec:theory}

\begin{proposition}[Anytime reduction to the implemented batch procedure]
\label{prop:anytime}
Let the stream terminate at $t = n$, let the online dependency graph equal
the graph a batch evaluator would extract from the completed log, and let
\OtapS{} perform one cold-started converged solve with the frontier mask
removed and marginals renormalized to the batch convention. Then the
finalization computation --- the same deterministic implemented
procedure (clipped GW linearization and surrogate diagnostic,
Sec.~\ref{sec:cg}; we say \emph{procedure} rather than
\emph{objective} because a clipped gradient need not be the gradient of
any objective), constraint set, parameters $(\theta, \lambda,
\varepsilon)$, references, node costs, initialization, and stopping
schedule, with age weights disabled (batch convention) --- is identical
to a batch evaluator running the same released implementation on the
completed trajectory, and the deterministic solver returns the same
score.
\end{proposition}
\begin{proof}[Proof sketch]
With the mask removed and marginals renormalized, every input to the
solver coincides with the batch invocation, and a deterministic solver
run from the same canonical initialization produces the same iterates.
\end{proof}
\begin{remark}
Three qualifications are inherent. The reduction is to the
\emph{implemented} procedure, not to an exact minimization of
Eq.~\eqref{eq:ufgw}: the finalization path uses the same clipped
gradient and surrogate diagnostic as the fast path, so equality with an
exact-UFGW batch score would additionally require an unclipped, exact-KL
final solve, which we do not currently perform. The UFGW objective is
nonconvex, so even then the guarantee would concern the solver's output,
not a global optimum. And the reduction is conditional on online
dependency extraction agreeing with the batch-extracted graph, which
retroactive repair targets but cannot guarantee
(Sec.~\ref{sec:structure}). Finally, a definition: the empirical
$\le 0.013$ figure (Sec.~\ref{sec:controlled}) is the
\emph{end-of-stream discrepancy} --- the running fast-path diagnostic
over the stream against the eventual finalization output. It is a
consistency check across the whole run, not a per-prefix tracking gap;
the proper tracking experiment (fast-path versus converged solve on the
\emph{same} prefix, mask, marginals, and references) is well-defined and
left to the released artifact.
\end{remark}

\begin{proposition}[Relabeling invariance of the masked objective]
\label{prop:perm}
Fix a covered sub-DAG, a mask $\nub^{(t)}$, and structural weights that
are order-independent (age weights matched across comparisons or
disabled). Then the masked objective, and its exact minimizer set from a
permutation-equivariant initialization, are invariant to
graph-preserving relabelings of the prefix: the satisfied set, active
set $\actset$, and masked marginal $\nub^{(t)}$ are functions of the
covered sub-DAG alone, and members of an executable antichain carry full
mass simultaneously, so no order among them is preferred.
\end{proposition}
\begin{proof}[Proof sketch]
Satisfied status depends on accumulated coupling mass to reference nodes,
which --- by the permutation invariance of the batch cost and structure
constructions --- depends only on the covered sub-DAG. $\actset$ and
$\nub^{(t)}$ are functions of the satisfied set alone
(Eq.~\ref{eq:mask}).
\end{proof}
\begin{remark}
Full \emph{online serialization} invariance is deliberately not claimed
as a theorem: even under exact per-event optimization, two
serializations that expose the same covered sub-DAG can assign different
\emph{ages} to its nodes (different age-weighted structural costs) and
arrive with different warm-start couplings, both path-dependent. What
holds by construction is the proposition above; what holds at
termination is exact, because the finalization path disables age weights
(batch convention, Sec.~\ref{sec:structure}) and runs the cold-started
procedure of Prop.~\ref{prop:anytime}, removing both sources of path
dependence. Mid-stream serialization invariance is therefore an
\emph{empirical} property: Scenario~B of Table~\ref{tab:faults}, where a
reordered run tracks and finalizes to the identical score, is its
validation.
\end{remark}

\begin{proposition}[Streaming tolerance band]
\label{prop:band}
Let a redundant leaf step (no downstream dependencies, arbitrary attributes)
be inserted at any position in the stream, and suppose the masked
reference marginal $\nub^{(t)}$ and the existing prefix structure are
unchanged by the insertion (the leaf neither satisfies a reference node
nor adds edges among existing nodes). Under the constant per-node mass
of Eq.~\eqref{eq:mass}, the increase in the optimal objective is at most
$\lambda/\mbar$ --- the generalized-KL cost of annihilating the node's mass
$\mu = 1/\mbar$ --- so the aggregate perturbation decays as $O(1/t)$
relative to total transported mass; its node leakage $\leak$ lies in
$[0,1]$ by definition (Eq.~\ref{eq:leak}) and attains $1$ exactly when
the node's mass is fully destroyed.
\end{proposition}
\begin{proof}[Proof sketch]
Extending the previous optimal coupling by an all-zero row for the inserted
leaf is feasible: the linear and quadratic terms gain nothing from a zero
row, the entropy contribution of zero entries vanishes, and the row-marginal
penalty gains $\lambda\,\mathrm{KL}(0 \,\|\, 1/\mbar) = \lambda/\mbar$
under the generalized KL of Eq.~\eqref{eq:ufgw}. The new optimum is bounded
above by this feasible point. Note this is an upper bound, not an identity:
the optimizer may rearrange mass on existing rows, and the leaf may couple
partially rather than be destroyed. The bound concerns the exact
objective Eq.~\eqref{eq:ufgw}; for the implemented $L_1$ mass surrogate
the same feasible-extension argument yields the same constant
$\lambda/\mbar$ (a zero row shifts the total-mass gap by exactly
$1/\mbar$). This is the streaming counterpart of the redundant-leaf
bound of~\cite{barazandeh2026otap} with $r=1$ and constant node mass.
\end{proof}

\begin{proposition}[Bounded influence of young structure]
\label{prop:fp}
Under age weights~\eqref{eq:age}, the \emph{direct} structural cost of
row $i$ in the linearized objective is zero at its emission event, and at
most the fraction $(t - t_i)/h_{\mathrm{age}} < 1$ of its full magnitude
for the following $h_{\mathrm{age}}$ events --- unless one of its
artifacts is consumed earlier, in which case its structural row is exact
along the evidenced edges and the discount is lifted. Node $i$'s coupling
may still be \emph{indirectly} affected through the shared Sinkhorn duals
and column marginals; the bound governs the structural cost entries, not
the full transport solution.
\end{proposition}
\begin{proof}[Proof sketch]
Immediate from $\omega_i(t_i) = 0$, the linear ramp of
Eq.~\eqref{eq:age}, and the row-scaling of Eq.~\eqref{eq:gwgrad}; the
clamp on consumption replaces an assumption by an observation. A hard
zero-weight window of $h_{\mathrm{age}}$ events would strengthen this to
strict silence at the cost of delaying true structural detections by the
same window; we prefer the linear ramp and note that in practice
structural verdicts additionally require the grace-TTL of the exploration
lifecycle (Sec.~\ref{sec:verdicts}) to expire.
\end{proof}

\paragraph{Complexity.}
Per event and reference: cost-row extension $O(m d)$; $\Dm^P$ update
$O(t\cdot\mathrm{deg})$ amortized (bitset ancestors); GW gradient
$O(t^2 m + t m^2)$; each Sinkhorn iteration $O(tm)$. Memory:
$O(t^2 + K t m + K m^2)$ numeric entries (the prefix structure matrix
$\Dm^P$, per-reference couplings, and reference structure matrices) plus
the $O(t^2)$-\emph{bit} ancestor representation. Measured on an Apple M5 Max core, the
fast path runs at p50 0.85\,ms / p95 1.6\,ms per event at
$t{=}m{=}40$, $K{=}3$ (p50 0.5 / p95 0.7\,ms at $t{=}m{=}20$; p50 1.6 /
p95 3.9\,ms at $t{=}80$); the escalated solve matches the batch-solve
runtime ($\approx$30\,ms) and is invoked only ahead of interventions. For long
horizons ($t \gtrsim 100$), a natural extension is to collapse completed
and fully consumed subgraphs into super-nodes, applying the hierarchical
recursion of~\cite{barazandeh2026otap} online; we sketch this but do not
implement or evaluate it.

\section{Experiments}
\label{sec:experiments}

We evaluate \OtapS{} in two regimes: (i) controlled fault injection with
ground truth by construction (Sec.~\ref{sec:controlled}), and (ii) streaming
replay of 2{,}294 real SWE-agent trajectories from the
\texttt{sailplane/swe-agent-trajs}
dataset~(Secs.~\ref{sec:early}--\ref{sec:boundary}).
All code, configs, and result JSONs are released with the paper.

\subsection{Controlled fault injection}
\label{sec:controlled}

A six-step reference task (\emph{search} $\to$
\{\emph{read log} $\parallel$ \emph{check schema}\}
$\to$ \emph{edit} $\to$ \emph{test} $\to$ \emph{commit}) with synthetic unit
embeddings, paraphrase noise on valid steps, and eight injected scenarios
covering both required detections and required silences.

\begin{table}[t]
\centering\small
\setlength{\tabcolsep}{2pt}\footnotesize
\begin{tabular}{@{}llccc@{}}
\toprule
Scenario & Expected & Fired & Lat. & Layer \\
\midrule
A\ \ clean in-order run          & silence & silence \checkmark & --- & --- \\
B\ \ valid reordering            & silence & silence \checkmark & --- & L2 \\
C\ \ strict loop (2-step cycle)  & \textsf{LOOP} & \textsf{LOOP} & 0 & L1 \\
D\ \ hallucination streak        & \textsf{OFF-TRK} & \textsf{OFF-TRK} & 2 & L2 \\
E\ \ causal inversion            & \textsf{C-INV} & \textsf{C-INV} & 0 & L2 \\
F\ \ retry after tool error      & silence & silence \checkmark & --- & L1 \\
G\ \ parallel branch (2 steps)   & silence & silence \checkmark & --- & L2 \\
H\ \ irreversible w/o prereqs    & \textsf{BLOCK} & \textsf{BLOCK} & 0 & L3 \\
\bottomrule
\end{tabular}
\caption{Controlled fault-injection results
(\textsf{OFF-TRK}\,=\,\textsf{OFF\_TRACK},
\textsf{C-INV}\,=\,\textsf{CAUSAL\_INVERSION}). Latency in events after
anomaly onset (target $\le 2$); the last column attributes each behavior
to the responsible layer. Scenario~B finalizes to the in-order score exactly
($0.913 = 0.913$), the empirical face of Prop.~\ref{prop:perm}. In D, the
two grace-window events preceding escalation surface as
\textsf{EXPLORING}, not silence.}
\label{tab:faults}
\end{table}

We are explicit about attribution: loops (C) and retry exemptions (F) are
handled by the reference-free L1 heuristics, and the irreversible-action
gate (H) by L3 schema checks --- neither exercises optimal transport. The
transport layer is responsible for the leakage-driven hallucination
escalation (D), the premature-step inversion check (E), the reordering
silence and score equality (B), and the parallel-branch tolerance (G).
Table~\ref{tab:faults}: all eight scenarios behave as designed. Per-node
leakage separates cleanly --- injected hallucinated steps reach
$\leak \ge 0.66$ while valid steps stay $\le 0.24$, comfortably
bracketing $\theta_{\mathrm{leak}} = 0.6$. At $K{=}3$ references,
$m{=}40$, $t \le 40$, the fast path measures p50 0.85\,ms / p95
1.6\,ms per event (Apple M5 Max core).

\paragraph{Gate scenarios --- including the bug the tests caught.}
Scenario~H shows only that an \emph{invalid} irreversible action is
blocked; a gate must also \emph{allow} valid ones. Four dedicated gate
tests (Table~\ref{tab:gate}) check both directions, plus a false-block
sweep over 50 benign runs ending in a valid irreversible call.

\begin{table}[t]
\centering\small
\setlength{\tabcolsep}{3pt}
\footnotesize
\begin{tabular}{@{}llll@{}}
\toprule
 & Scenario & Expected & Result \\
\midrule
G-a & valid irrev., prereqs + ref & allow & allow \checkmark \\
G-b & missing prereq artifact & block & block \checkmark \\
G-c & schema-valid, no ref match & policy & block (conserv.) \\
G-d & no schema, no references & policy & allow (no check) \\
\midrule
\multicolumn{2}{@{}l}{False blocks, valid irrev.\ calls} &
\multicolumn{2}{l}{0/50 (95\% bound $\approx$6\%)} \\
\bottomrule
\end{tabular}
\caption{L3 gate tests --- a small validation suite, not a reliability
claim. G-c/G-d record implemented policy: with references present the
gate requires both the schema check and the proposal match, checked
against the current best (lowest-loss) reference only (conservative);
in the absence of both schemas and references the gate operates
\emph{allow-by-default} --- an explicit deployment policy, not a safety
guarantee.}
\label{tab:gate}
\end{table}

We report these tests with a disclosure that argues for their existence:
G-a initially \emph{failed}. The first gate implementation read the
frontier cached at the start of the previous event --- computed from
$\Tm_{t-2}$, one event stale for a pre-execution proposal --- and
consequently blocked 50/50 \emph{valid} irreversible calls while
scenario~H continued to pass (blocking an invalid call requires no
frontier at all). The fix recomputes the frontier from the current
coupling inside the gate; after it, G-a allows, G-b/H still block, the
false-block rate is 0/50, and all prior results are unaffected (the
SWE-bench corpus contains no irreversible-annotated tools, so the gate
never fires there). A block-only test suite would never have revealed
this.

\paragraph{Component attribution on real traces.}
A reviewer-requested ablation (Table~\ref{tab:component}) runs the same
fault types on real SWE-agent traces (40 benign resolved traces, 30
injected traces per fault, trace-level severe-flag detection):

\begin{table}[t]
\centering\small
\setlength{\tabcolsep}{2.5pt}
\begin{tabular}{@{}lccccc@{}}
\toprule
Variant & Loop & Halluc & Invert & Stall & FPR \\
\midrule
Full system          & 30/30 & 12/30 & 12/30 & 30/30 & 13/40 \\
True no-mask ($\nu^{(t)}{=}\nu$) & 30/30 & 13/30 & 12/30 & 30/30 & 14/40 \\
No GW ($\theta{=}0$) & 30/30 & 21/30 & 25/30 & 30/30 & 29/40 \\
L1-only (no refs)    & 30/30 & 18/30 & 19/30 & 30/30 & 31/40 \\
Prereq-rule baseline & 30/30 & 26/30 & 24/30 & 30/30 & 32/40 \\
Frozen frontier ($\kappa{=}100$) & 30/30 & 30/30 & 26/30 & 30/30 & 38/40 \\
\bottomrule
\end{tabular}
\caption{Component study on real SWE-agent traces at the shipped
verdict thresholds, reported as raw counts (detected/injected per fault;
false-positive traces over $n{=}40$ benign traces; trace-level
$\ge$1-severe-flag criterion --- the deployed abort policy thresholds
flag \emph{density} instead, Table~\ref{tab:abort}, hence its far lower
false-stop rate). Wilson 95\% intervals at these sample sizes span
roughly $\pm 15$ percentage points (e.g., 12/30 = 40\%
$[24.6\%, 57.7\%]$; 13/40 = 32.5\% $[20.1\%, 48.0\%]$), so only the
large contrasts are meaningful. ``True no-mask'' removes the frontier
mask entirely; ``frozen frontier'' ($\kappa{=}100$, nothing ever
satisfied) is a \emph{mis-set} mask, included to show that a wrongly
configured mask is worse than none.}
\label{tab:component}
\end{table}

Three readings are honest here, and we distinguish them. First,
attribution: loops and stalls are carried entirely by the reference-free
L1 heuristics --- every variant, including the rule baseline, matches on
them. Second, at the shipped threshold the \emph{quadratic structure
term selects a lower-FPR operating point}: removing the mask changes
almost nothing (13/40 $\to$ 14/40 false-positive traces at essentially
identical recall), while removing the GW term more than doubles the
benign false-alarm count (29/40), to nearly the level of running no
transport at all (L1-only 31/40, rule baseline 32/40). Third, the recall
side remains an \emph{operating-point} contrast, not a dominance result:
at the shipped thresholds the full system sits at a conservative point
(12/30 recall on hallucinations and inversions) while the no-GW and rule
variants sit at high-recall, high-FPR points (21--26/30 recall at
29--32/40 FPR).
Threshold sweeps with per-point counts and Wilson intervals (released
with the code; $n{=}120$ benign, $100$ per fault) quantify this: loop
and stall AUPRC is $0.70$--$0.74$ for every variant (base rate $0.45$),
while hallucination and inversion AUPRC sits at $0.44$--$0.47$ ---
essentially the base rate --- for \emph{all} variants, i.e., the
precision--recall curves are indistinguishable and only the shipped
operating point differs. We therefore claim the FPR contrast at deployed
thresholds, not curve-level dominance.

\paragraph{Scaled campaign.}
Beyond the eight unit scenarios, a randomized campaign (40 random-DAG
tasks $\times$ 7 conditions, seeded) yields: loop recall 100\% at
detection latency 0; hallucination recall 100\% at latency 2 (the grace
window); causal-inversion recall 94.1\% at latency 0 via the
premature-step check; and 0 false interventions over 200 benign runs
that include valid reorderings, retries, and parallel exploration (a
one-sided 95\% upper bound of 1.5\% on the false-intervention rate). The
end-of-stream discrepancy --- the running fast-path diagnostic against
the eventual finalization output, as defined in the remark after
Prop.~\ref{prop:anytime} --- never exceeds 0.013 across the campaign
(and is 0 after finalization); per-event runtime spans 0.2--1.6\,ms.

\subsection{SWE-bench setup}
\label{sec:swe-setup}

\paragraph{Dataset.}
The \texttt{sailplane/swe-agent-trajs} dataset contains 2{,}294 SWE-agent
trajectories with tool-call logs, \texttt{exit\_status}, and
\texttt{instance\_cost}; 2{,}288 parse into usable event streams
(430 resolved, 1{,}858 unresolved) and 6 malformed trajectories are
dropped. Each trace is streamed through \OtapS{} event-by-event; tool
embeddings use \texttt{all-MiniLM-L6-v2}.

\paragraph{Instance-disjoint split.}
Traces are grouped by SWE-bench instance ID. Instances are randomly
partitioned (seed=42) into calibration (228 instances), reference (5), and
evaluation (remaining). No instance appears in more than one partition,
preventing information leakage through shared task structure. Reference
instances are also excluded from evaluation.

\paragraph{Calibration.}
$\lambda$ is selected on the calibration set only, by AUROC over a grid
$\{0.01, 0.05, 0.1, 0.2, 0.3, 0.5, 1.0\}$; the selected $\lambda = 0.3$
(calibration AUROC 0.743) is frozen for all subsequent experiments. The
algorithm is otherwise frozen: observation embeddings enter monitoring
signals only, never the transport cost ($\gamma_{\mathrm{eff}} = 0$,
cf.\ the negative result in Sec.~\ref{sec:boundary}), content divergence is
mass-weighted, CG averaging is kept at its default ($\gamma = 0.7$,
Sec.~\ref{sec:cg}), and no outcome-anchoring signals are used.

\paragraph{Baselines.}
The baseline pool, drawn on per-table as relevant: cosine similarity
(mean reference--trace embedding distance), \texttt{n\_steps} (trace
length), \texttt{grep\_frac} (fraction of search-tool calls), and Markov
surprise (fraction of unique tool bigrams). Bootstrap 95\% CIs
($n=1{,}000$) are reported on headline numbers.

\paragraph{Evaluation protocol.}
The \emph{primary} protocol, used for every headline number, is a single
fixed held-out evaluation: all evaluation-partition traces with at least
$k$ steps, every method scored on exactly the same traces, and
bootstrap 95\% CIs ($n{=}1{,}000$) on AUROC and on the difference
vs.\ cosine. Since this corpus has exactly one trace per instance
(2{,}288 unique instances, 2{,}288 traces), instance-clustered and
trace-level bootstrap coincide. AUROC orientation: the resolved class is
positive; the \OtapS{} score $s$ (higher = better aligned) is used
directly, and each baseline is oriented so that higher predicts
resolution (length- and flag-count-type signals negated). A secondary
protocol --- five seeds
redrawing balanced $\approx$200-trace samples, with the partition,
references, and $\lambda$ fixed --- is reported in
Appendix~\ref{app:multiseed} for robustness.

\subsection{Early warning and cost savings}
\label{sec:early}

The headline experiment asks: can \OtapS{} identify failing traces
\emph{early enough} to save compute by aborting? Two sub-questions:
(a) at what partial-trajectory length does process alignment outperform
content similarity, and (b) can a practical abort policy keep false stops
bounded?

\paragraph{Partial-trajectory discrimination.}
For each trace with $\ge k$ steps, we truncate to the first $k$ steps and
compute the \OtapS{} score, cosine similarity, and baselines. Table~\ref{tab:partial}
reports AUROC for predicting resolved vs.\ unresolved at each checkpoint.

\begin{table}[t]
\centering\small
\setlength{\tabcolsep}{2pt}\footnotesize
\begin{tabular}{@{}rcccccl@{}}
\toprule
$k$ & \OtapS{} & LinUOT & Cosine & \texttt{n\_st.} & GED &
$\Delta$ vs.\ cosine \\
\midrule
5  & .612 & .595 & .552 & \textbf{.622} & .574 & $+.060\,[.037,.083]$ \\
8  & \textbf{.631} & .628 & .574 & .613 & .576 & $+.057\,[.033,.080]$ \\
10 & \textbf{.622} & .614 & .589 & .598 & .552 & $+.033\,[.007,.057]$ \\
15 & \textbf{.635} & .626 & .633 & .580 & .551 & $+.002\,[\!-.022,.027]$ \\
20 & \textbf{.648} & .638 & .646 & .554 & .566 & $+.001\,[\!-.029,.031]$ \\
\bottomrule
\end{tabular}
\caption{Primary protocol: AUROC at partial trajectory (first $k$
steps), every evaluation-partition trace with $\ge k$ steps
($N = 2{,}055$ at $k{=}5$ down to $1{,}405$ at $k{=}20$), all methods
scored on the same traces, instance-level bootstrap 95\% CIs
($n{=}1{,}000$) on the \OtapS{}$-$cosine difference. The $\Delta$ CI
excludes zero for $k \le 10$.}
\label{tab:partial}
\end{table}

Table~\ref{tab:partial} is the primary evaluation. Baselines beyond
those of Sec.~\ref{sec:swe-setup}: LinUOT is linear unbalanced OT
($\theta{=}0$, no structure term); GED is a graph-edit-distance proxy
(label-multiset distance plus edge overlap on the tool-labeled
dependency graph); a Laplace-smoothed trigram model and the
prerequisite-rule detector from the component study are evaluated but
weaker throughout (trigram declining from $.596$ to $.538$; rule at or
below chance) and omitted from the table for space. Four observations.
First, the $\Delta$ CI excludes zero at $k \le 10$, and the advantage is
gone by $k{=}15$, where cosine catches up ($.633$ vs.\ $.635$): the
claim is an \emph{early-trajectory} advantage, precisely the regime an
in-flight monitor operates in. Second, at the earliest checkpoint trace
length alone is the single best signal ($\texttt{n\_steps}$ $.622$ at
$k{=}5$); \OtapS{} leads every content- and structure-aware method there
but not the length prior --- its clearest regime is $k{=}8$--$10$, after
the length prior fades and before content accumulates. Third, LinUOT
tracks \OtapS{} within $0.003$--$0.010$ AUROC everywhere: the quadratic
term contributes essentially nothing to \emph{ranking} on this corpus,
consistent with its role as a false-alarm suppressor rather than a
ranker (Table~\ref{tab:component}); the mask-free variant is
AUROC-identical to the full system, confirming the mask's null ranking
effect here as well. Fourth, the balanced multi-seed protocol
(Appendix~\ref{app:multiseed}) agrees in direction with somewhat larger
deltas; we cite the primary protocol's numbers everywhere.

\paragraph{Low-FPR operating points.}
At $k{=}8$, \OtapS{} achieves TPR 6.6\% / 23.2\% / 33.3\% at FPR
1\% / 5\% / 10\%, compared to cosine 1.5\% / 9.6\% / 20.7\% and
\texttt{n\_steps} 1.0\% / 5.6\% / 13.6\%. The advantage is consistent
across all three operating points.

\paragraph{Abort policy.}
A na\"ive ``abort at first flag'' policy is useless: nearly every trace
--- resolved or not --- emits at least one \textsf{EXPLORING} flag early
on (false-stop rate 99\%). We restrict to \emph{severe} flags ---
\textsf{LOOP}, \textsf{OFF\_TRACK}, \allowbreak
\textsf{CAUSAL\_INVERSION}, \textsf{BLOCK}, \textsf{STALLED} --- and
trigger abort only when the flag density in a rolling window of 6 steps
exceeds a threshold.

\begin{table}[t]
\centering\small
\setlength{\tabcolsep}{3pt}
\begin{tabular}{@{}ccccc@{}}
\toprule
Threshold & True stop & False stop & Cost saved & Prec.\ (dep.) \\
\midrule
0.20 & 33.3\% & 29.3\% & 28.0\% & 83\% \\
0.40 & 27.3\% & 26.7\% & 20.9\% & 81\% \\
\textbf{0.60} & \textbf{23.3\%} & \textbf{20.7\%} & \textbf{17.9\%} & \textbf{83\%} \\
0.70 & 20.7\% & 19.3\% & 15.6\% & 82\% \\
\bottomrule
\end{tabular}
\caption{Abort threshold sweep on severe flag density ($n{=}150$ per
class, balanced). Precision is reported at the corpus base rate (81\%
of traces unresolved), which is the deployment-relevant figure; on the
balanced sample it is necessarily near 50\%.}
\label{tab:abort}
\end{table}

At threshold 0.60 (Table~\ref{tab:abort}), 23.3\% of unresolved traces
are stopped early at a 20.7\% false-stop rate, yielding 17.9\% cost
savings on the unresolved population. On an earlier, independent draw
of the profile sample the same threshold gave 25.3\% true stops at
12.7\% false stops and 17.0\% savings: the \emph{savings} figure is
stable across trace mixes ($\approx$17--18\%), while the false-stop
rate varies (13--21\%) and should be validated on deployment traffic.
Because the corpus base rate is heavily skewed toward failure (81\%
unresolved), the balanced-sample false-stop rates translate to
$\approx$83\% precision in deployment: five of six aborted runs would
indeed have failed. One protocol disclosure: the full threshold sweep is
reported, and 0.60 is an \emph{illustrative} threshold selected for
exposition after the sweep --- not a recommended deployment setting,
since it was not chosen on held-out data. What we claim is a
\emph{candidate} policy: $\approx$18\% savings at $\approx$20\% balanced
false stops at the illustrated point, not production readiness. A
deployment should fix its
threshold on separate calibration traffic; the same applies to the
fixed-FPR operating points reported in this paper, which are ROC read-outs on the
evaluation sample rather than pre-registered thresholds.

\paragraph{Rival abort policies.}
We compare against four baselines on the same 150+150 profiles:
\emph{step-cutoff} (abort when step count exceeds $N$),
\emph{budget-cutoff} (abort when cumulative cost exceeds a fraction of
the median resolved-trace cost), \emph{cosine-cutoff} (abort when
running cosine-to-reference drops below a threshold), and a
\emph{hybrid} (flag-density OR step-cutoff). The honest finding is
that \textbf{the savings-vs-false-stop frontier alone does not select
a winner}. At a matched $\approx$20\% false-stop rate, flag-density
(17.9\%) and an interpolated step-cutoff ($\approx$16.5\%) save
comparable cost; at $\approx$29\% false stops, blunt early policies
save \emph{more} (cosine-cutoff 37.9\%, budget-cutoff 31.9\%, vs.\
flag-density 28.0\%), because savings track how early a policy fires
and nothing fires earlier than a cutoff that barely looks at the trace
(cosine-cutoff's median abort is step 6, vs.\ 33--44 for flag density).
The hybrid tracks the step-cutoff frontier and adds nothing.

What separates the policies is \emph{operability}, on two axes. First,
threshold stability: cosine-cutoff's frontier is pathologically
sensitive --- moving its threshold from 0.40 to 0.55 swings the
false-stop rate from 3.3\% to 28\% to 73\% to 95\%, so a miscalibration
of 0.05 turns it from inert into a kill-everything switch --- whereas
flag density degrades gracefully across its entire range (29\% to 19\%
false stops over thresholds 0.2--0.7). Second, actionability: a cutoff
emits a bare kill signal, while a severe-flag trigger fires with a
\emph{verdict} (\textsf{LOOP} at step $t$, \textsf{STALLED} since step
$t{-}w$) that supports graduated responses --- replan, inject guidance,
escalate to a human --- of which aborting is only the bluntest. We
therefore present the abort experiment as evidence that the verdict
stream supports a \emph{candidate} operating policy that is reasonably
stable on this corpus, not that it dominates the savings frontier. A full expected-utility analysis ---
pricing an aborted success at its completion cost plus lost outcome value
rather than by its compute alone --- is needed before any of these
policies is production-ready, and is future work.

\subsection{Process anomaly detection}
\label{sec:anomaly}

\paragraph{Exit-status detection (L1-only).}
SWE-agent traces exit with \texttt{cost\_limit\_exceeded} (401 traces,
combining \texttt{exit\_cost} and \texttt{submitted (exit\_cost)}) or
clean submission (1{,}868); the remaining 19 traces carry other exit
statuses (\texttt{None} or \texttt{early\_exit}) and are excluded from
this binary comparison. We evaluate on a balanced 50+50 sample. This
is a pure process anomaly: cost-limit traces are agents that ran out of
budget, not agents that chose wrong actions. Table~\ref{tab:exit} shows
that \OtapS{} score (0.979) and flags (0.977) dominate cosine (0.557)
and statistically match \texttt{n\_steps} (0.978). We do not claim an
edge over trace length here --- at full-trace granularity, cost-limit
detection \emph{is} essentially length, and a step counter is equally
available online. The relevant distinction is \emph{when} and \emph{how}
the signal is usable: a raw step count offers no principled stopping
rule mid-trace (any length cutoff is a per-task guess), whereas the
severe-flag density that drives Table~\ref{tab:abort} is a calibrated,
task-agnostic trigger that identifies the loop/flounder behavior
\emph{producing} the cost overrun while budget remains to save.

\begin{table}[t]
\centering\small
\begin{tabular}{@{}lc@{}}
\toprule
Method & AUROC \\
\midrule
\OtapS{} score          & \textbf{0.979} \\
\OtapS{} flags          & 0.977 \\
\texttt{n\_steps}       & 0.978 \\
Markov surprise         & 0.905 \\
\OtapS{} leak           & 0.958 \\
Cosine                  & 0.557 \\
\texttt{grep\_frac}     & 0.447 \\
\bottomrule
\end{tabular}
\caption{Exit-status detection: cost-limit vs.\ clean submission
(balanced 50+50 sample). \OtapS{} matches \texttt{n\_steps}; the value
added is the calibrated mid-trace trigger (Table~\ref{tab:abort}), not
the full-trace AUROC.}
\label{tab:exit}
\end{table}

At low FPR operating points (balanced 100+100 sample), \OtapS{} achieves
TPR 37\% / 81\% / 92\% at FPR 1\% / 5\% / 10\%, matching
\texttt{n\_steps} (33\% / 82\% / 92\%) and far exceeding cosine
(0\% / 3\% / 9\%). The severe-flag density trigger (th=0.60) fires on
90\% of cost-limit traces, at a median of 58.6\% of the trace elapsed
--- i.e., with $\approx$41\% of the eventual trajectory \emph{steps}
still ahead at detection. We report the step fraction rather than a
dollar fraction because the dataset logs per-trace, not per-step, cost;
the two coincide only under uniform per-step cost, and empirically the
remaining-step fraction correlates only weakly and non-significantly
with total per-trace cost (Pearson $r = -0.19$, $p = 0.06$), so a
dollar conversion is not licensed by the data.

\paragraph{In-distribution splice injection.}
To test structural sensitivity beyond outcome labels, we inject 3-step
segments from real unresolved traces into resolved traces at random
positions (60 pairs). Unlike out-of-distribution hallucinations
(weather-API or infrastructure calls, which any embedding-distance
method detects trivially), these donor steps use the same tool
vocabulary as the host trace. The honest result is that such splices sit
near the detection floor on this corpus: AUC(score)~=~0.546,
AUC(flags)~=~0.540, cosine 0.516, and the per-trace leakage delta
between spliced and clean versions is not significant
($p = 0.93$). Every spliced trace raises at least one flag, but
unspliced host traces of comparable length flag at nearly the same rate,
so trace-level ``detection rate'' is not a meaningful metric here and we
do not report it as one. Three real SWE-agent steps borrowed from another
debugging session are, locally, a plausible continuation --- separating
them requires task-level semantics that a process monitor with
task-generic references does not have. We report this as a
characterization of the method's floor on homogeneous real traces, in
deliberate contrast to the controlled setting of
Sec.~\ref{sec:controlled} where injected anomalies violate the reference
structure and are caught exactly.

\paragraph{Fault injection into real traces.}
Injecting the controlled fault types into real resolved traces, flag-based
AUROC is strongest on loops (0.699) and stalls (0.646), where structural
repetition is pronounced; in-distribution hallucinations (0.572) and
inversions (0.496) are near the floor for the same reason as the splice
experiment --- on homogeneous 30-step debugging traces, a locally
plausible wrong step is indistinguishable from exploration without
task-level semantics. Loops and stalls are also precisely the fault
modes that dominate real cost-limit failures (Table~\ref{tab:exit}),
which is why the abort policy of Table~\ref{tab:abort} works despite
the per-fault AUROC floor on the harder categories.

\subsection{The $\lambda$ finding: relaxation strength is the detector}
\label{sec:lambda}

Our controlled validation initially \emph{failed} on hallucination
detection. In entropic unbalanced OT with exponent
$\eta = \lambda/(\lambda+\varepsilon)$ applied per scaling update, the
decoupled single-cell fixed point of Eq.~\eqref{eq:ufgw} --- both
marginals KL-relaxed --- retains mass
$t^{*} \propto \exp\!\big(-M/(2\lambda+\varepsilon)\big)$ for a cell with
cost $M$ (set the derivative of
$Mt + \varepsilon\, t(\log t - 1) + \lambda\, t\log(t/\mu_i) +
\lambda\, t\log(t/\nu_j)$ to zero). This is a one-cell approximation --- the full solution couples
cells through the shared marginals --- and we use it as a design
heuristic, validated empirically via the $\lambda$ sweep.
With the batch-evaluation default $\lambda = 1$ and node costs normalized to $[0,1]$,
a hallucinated step with $M \approx 1$ retains a large fraction of its mass:
the solver \emph{prefers transporting garbage over destroying it}, and
leakage-based diagnostics never fire. Online, where the \emph{marginals} are
the signal, $2\lambda + \varepsilon$ acts as a soft cost threshold below
which mass is destroyed rather than moved. The calibrated value falls
below the batch default by roughly $3$--$10\times$ across our two
settings ---
$\lambda = 0.1$ on the controlled tasks with synthetic embeddings,
$\lambda = 0.3$ on SWE-bench with \texttt{all-MiniLM-L6-v2}
(Sec.~\ref{sec:swe-setup}) --- and with it, clean separation emerges.
The dependence on the embedding model confirms that $\lambda$ must be
recalibrated whenever the encoder (and hence the cost scale) changes.

\subsection{Outcome correlation: a boundary study}
\label{sec:boundary}

\OtapS{} is a process monitor, not an outcome predictor. We report the
correlation between \OtapS{} scores and the binary resolved/unresolved
label to characterize the boundary of what process alignment can and
cannot tell us.

\begin{table}[t]
\centering\small
\begin{tabular}{@{}lc@{}}
\toprule
Signal & AUROC \\
\midrule
\OtapS{} score       & 0.581 \\
\OtapS{} leak        & 0.550 \\
\OtapS{} flags       & 0.577 \\
Cosine               & 0.643 \\
\texttt{n\_steps}    & 0.576 \\
Markov surprise      & 0.573 \\
Cost                 & 0.575 \\
\texttt{grep\_frac}  & 0.525 \\
\bottomrule
\end{tabular}
\caption{Resolved vs.\ unresolved (50+50 balanced), one row per
trajectory-level signal. Bootstrap 95\% CI:
\OtapS{} score $[0.466, 0.689]$, cosine $[0.531, 0.759]$ --- overlapping.}
\label{tab:outcome}
\end{table}

On a balanced sample (50 resolved, 50 unresolved), no method reliably
separates outcomes (Table~\ref{tab:outcome}): \OtapS{} reaches 0.581,
cosine 0.643, with widely overlapping bootstrap CIs. Point estimates are
unstable at this sample size --- an independent draw of the evaluation
sample yields \OtapS{} 0.666 vs.\ cosine 0.657, reversing the ordering
--- so the only conclusion the data support is that neither process
alignment nor content similarity predicts patch correctness from the
trajectory alone. This is the expected boundary: both resolved and
unresolved SWE-agent traces follow the same
search$\to$read$\to$edit$\to$test template, and the difference between
success and failure lies in \emph{which line was edited}, which no
task-generic trajectory signal observes.

\paragraph{Length confound.}
When resolved/unresolved pairs are matched within $\pm 3$ steps (50
pairs, mean 30.6 steps in both classes), all trajectory-level signals
collapse toward chance: \OtapS{} score 0.526, cosine 0.540,
\texttt{n\_steps} 0.500 by construction. This confirms that the modest
full-trajectory outcome correlation is largely mediated by trace length
--- unresolved traces are longer on average --- and that \OtapS{} score
is length-sensitive by construction (more steps $\to$ more accumulated
transport cost). We report this as a delimiter, not a defect: the
monitor's job (Secs.~\ref{sec:early}--\ref{sec:anomaly}) is to catch the
runaway-length failure mode \emph{early}, and the length sensitivity is
precisely the mechanism by which it does so.

\paragraph{Design insight: observations belong in signals, not in the
cost.}
A natural extension is to make the transport cost \emph{effect-aware}:
blend an observation-alignment term into the node cost,
$C_{ij} \mathrel{+}= \gamma_{\mathrm{eff}}\,
d_{\cos}(\phi(e_i), \phi(e^R_j))$, so that a step that ran the right
tool but produced the wrong result couples poorly. Empirically this is a
negative result: at $\gamma_{\mathrm{eff}} = 0.3$ the coupling quality
degrades across ablations (on a held-out tool-calling corpus, mean
ablation AUROC drops from 0.825 to 0.715), because observations are
high-variance --- correct and incorrect executions of the same step both
produce textually different outputs every time, so the effect term
dilutes the action/tool signal that correctly drives the alignment. The
same information is useful \emph{downstream} of the coupling: a
mass-weighted content-divergence signal (does the observation match what
the reference saw at the coupled node?) contributes to struggling-agent
detection without touching the transport. The shipped default is
$\gamma_{\mathrm{eff}} = 0$, with content divergence reported as a
monitoring signal.

\section{Limitations and Future Work}
\label{sec:limits}

\textbf{Scope.} Where a formal validity checker exists, use the checker; the
oracle-study negative result of~\cite{barazandeh2026otap} delimits proximity
metrics, ours included, to open-ended domains without verifiers.
\textbf{Outcome prediction.} \OtapS{} does not predict whether an agent
will succeed --- it detects process anomalies that correlate with failure.
The length-controlled AUROC ($0.526$, with cosine at $0.540$ and every
other trajectory signal similarly near chance) makes this boundary
explicit for all task-generic signals, not only ours.
\textbf{In-distribution anomalies.} Short splices of real steps from
other tasks sit near the detection floor on homogeneous traces
(Sec.~\ref{sec:anomaly}); catching them requires task-level semantics
(e.g., instance-matched references or a verifier), which we consider the
most important open direction.
\textbf{Extraction.}
Online dependency extraction remains the binding constraint, as in the
post-hoc setting; age
weights and retroactive repair manage the noise, they do not eliminate it.
\textbf{False stops.} Even with flag-density thresholds, the illustrative
operating point interrupts 13--21\% of well-behaved traces depending on
the trace mix (up to one in five at the pessimistic end, mitigated in
deployment by the failure-skewed base rate). Deployments must weigh the cost
of premature termination against the savings from catching failures.
\textbf{Reference quality.} The SWE-bench experiments use a small reference
set (5 instances); larger and better-matched references may improve
discrimination.
\textbf{Structure term on homogeneous traces.} On SWE-agent traces,
which follow a narrow search--read--edit--test template, a
checkpoint-level ablation at $k{=}5$ and $k{=}10$ (5 seeds) shows
Full \OtapS{} exceeds the No-GW ($\theta{=}0$) variant by only
$+0.009$ and $+0.001$ respectively --- within sampling noise. The
quadratic structure term's value is established in the controlled
setting (Sec.~\ref{sec:controlled}) and is expected to emerge on
structurally diverse corpora (heterogeneous tool vocabularies, branching
workflows); on the homogeneous SWE-agent template it contributes little
beyond the fused attribute cost.
\textbf{Theory.} A formal bound on the streaming-CG tracking error in
terms of per-event drift and CG contraction is open.

\section{Conclusion}

\OtapS{} shows that structure-aware optimal transport --- until now a
post-hoc instrument --- can run inside the agent loop and act in real time,
without surrendering the properties that made it worth having. The quadratic
structure term survives through amortized conditional gradient; the prefix
problem is solved by frontier-masked marginals that keep order-freedom where
the DAG permits it and convert agent-side mass leakage into a real-time
signal; and every intervention rests on per-node transport marginals,
escalated to a converged solve before action.

On 2{,}294 SWE-agent trajectories, \OtapS{} provides early-trajectory
discrimination that outperforms cosine similarity where it matters most
--- the first ten steps ($+0.057$ AUROC at $k{=}8$ under the primary
held-out protocol, instance-level bootstrap 95\% CI $[+0.033, +0.080]$,
excluding zero through $k{=}10$) --- and
a candidate flag-density abort policy stops roughly one in four failing
runs, saving 17--18\% of compute on the failed population at $\approx$83\%
deployment-adjusted precision. Rival stopping policies can match or
exceed those savings at matched false-stop rates; what the verdict
stream adds to the stopping frontier is graceful degradation under
threshold miscalibration, an interpretable reason, and a graduated
response rather than a bare kill signal. The component
study clarifies what the transport machinery changes on real traces: not
recall --- L1 heuristics and even a prerequisite-rule baseline match or
exceed it there --- but the \emph{operating point}: at shipped
thresholds the quadratic structure term selects a point with less than
half the benign false-alarm rate of every reference-free variant (13/40
vs.\ 29--32/40 benign traces flagged; removing the frontier mask instead
shifts the count by one trace); establishing curve-level dominance
remains open, and we do not claim generally superior detection
precision. The precise summary of the real-trace evidence is this: on
homogeneous SWE-agent traces, \OtapS{} provides a modest early advantage
over content similarity, and its main structural contribution is
selecting a lower-false-alarm operating point rather than improving
ranking. On real traces, the demonstrated
strengths are loops, stalls, deviation from a known reference structure,
and schema-checkable prerequisite violations; we do not claim general
in-the-wild detection of hallucinations or causal inversions, whose
real-trace signal sits near the floor on homogeneous corpora. Cost-limit failures are detected at 0.979 AUROC, on
par with trace length post hoc but available as a calibrated trigger
mid-trace. We are equally transparent about the boundaries: on real,
homogeneous traces, in-distribution splices sit near the detection
floor, and no task-generic trajectory signal --- ours or any baseline
--- predicts patch correctness once length is controlled. \OtapS{} is a
process monitor, not an oracle, and we believe the distinction is worth
drawing clearly.

{\small
\bibliographystyle{unsrt}
\bibliography{references}

@article{barazandeh2026otap,
  author  = {Barazandeh, Babak and Majumdar, Subhabrata and Michailidis, George},
  title   = {{\textsc{Otap}}: Structure-Aware Optimal Transport for Evaluating Planning and Execution in Agent Trajectories},
  journal = {arXiv preprint arXiv:2607.17082},
  year    = {2026}
}

@inproceedings{cuturi2013sinkhorn,
  author    = {Cuturi, Marco},
  title     = {Sinkhorn Distances: Lightspeed Computation of Optimal Transport},
  booktitle = {Advances in Neural Information Processing Systems (NeurIPS)},
  year      = {2013}
}

@article{chizat2018unbalanced,
  author  = {Chizat, L{\'e}na{\"i}c and Peyr{\'e}, Gabriel and Schmitzer, Bernhard and Vialard, Fran{\c{c}}ois-Xavier},
  title   = {Scaling Algorithms for Unbalanced Optimal Transport Problems},
  journal = {Mathematics of Computation},
  volume  = {87},
  pages   = {2563--2609},
  year    = {2018}
}

@article{memoli2011gw,
  author  = {M{\'e}moli, Facundo},
  title   = {Gromov--{W}asserstein Distances and the Metric Approach to Object Matching},
  journal = {Foundations of Computational Mathematics},
  volume  = {11},
  pages   = {417--487},
  year    = {2011}
}

@inproceedings{vayer2019fgw,
  author    = {Vayer, Titouan and Courty, Nicolas and Tavenard, Romain and Chapel, Laetitia and Flamary, R{\'e}mi},
  title     = {Optimal Transport for Structured Data with Application on Graphs},
  booktitle = {Proceedings of the 36th International Conference on Machine Learning (ICML)},
  year      = {2019}
}

@article{peyre2019computational,
  author  = {Peyr{\'e}, Gabriel and Cuturi, Marco},
  title   = {Computational Optimal Transport},
  journal = {Foundations and Trends in Machine Learning},
  volume  = {11},
  number  = {5--6},
  pages   = {355--607},
  year    = {2019}
}

@inproceedings{zheng2023judging,
  author    = {Zheng, Lianmin and Chiang, Wei-Lin and Sheng, Ying and Zhuang, Siyuan and Wu, Zhanghao and Zhuang, Yonghao and Lin, Zi and Li, Zhuohan and Li, Dacheng and Xing, Eric P. and Zhang, Hao and Gonzalez, Joseph E. and Stoica, Ion},
  title     = {Judging {LLM}-as-a-Judge with {MT-Bench} and {Chatbot Arena}},
  booktitle = {Advances in Neural Information Processing Systems (NeurIPS)},
  year      = {2023}
}

@inproceedings{lightman2024verify,
  author    = {Lightman, Hunter and Kosaraju, Vineet and Burda, Yura and Edwards, Harri and Baker, Bowen and Lee, Teddy and Leike, Jan and Schulman, John and Sutskever, Ilya and Cobbe, Karl},
  title     = {Let's Verify Step by Step},
  booktitle = {International Conference on Learning Representations (ICLR)},
  year      = {2024}
}

@inproceedings{mialon2024gaia,
  author    = {Mialon, Gr{\'e}goire and Fourrier, Cl{\'e}mentine and Wolf, Thomas and LeCun, Yann and Scialom, Thomas},
  title     = {{GAIA}: A Benchmark for General {AI} Assistants},
  booktitle = {International Conference on Learning Representations (ICLR)},
  year      = {2024}
}

@inproceedings{zhou2024webarena,
  author    = {Zhou, Shuyan and Xu, Frank F. and Zhu, Hao and Zhou, Xuhui and Lo, Robert and Sridhar, Abishek and Cheng, Xianyi and Ou, Tianyue and Bisk, Yonatan and Fried, Daniel and Alon, Uri and Neubig, Graham},
  title     = {{WebArena}: A Realistic Web Environment for Building Autonomous Agents},
  booktitle = {International Conference on Learning Representations (ICLR)},
  year      = {2024}
}

@inproceedings{liu2024agentbench,
  author    = {Liu, Xiao and Yu, Hao and Zhang, Hanchen and Xu, Yifan and Lei, Xuanyu and Lai, Hanyu and Gu, Yu and Ding, Hangliang and Men, Kaiwen and Yang, Kejuan and Zhang, Shudan and Deng, Xiang and Zeng, Aohan and Du, Zhengxiao and Zhang, Chenhui and Shen, Sheng and Zhang, Tianjun and Su, Yu and Sun, Huan and Huang, Minlie and Dong, Yuxiao and Tang, Jie},
  title     = {{AgentBench}: Evaluating {LLMs} as Agents},
  booktitle = {International Conference on Learning Representations (ICLR)},
  year      = {2024}
}

@inproceedings{zhang2020bertscore,
  author    = {Zhang, Tianyi and Kishore, Varsha and Wu, Felix and Weinberger, Kilian Q. and Artzi, Yoav},
  title     = {{BERTScore}: Evaluating Text Generation with {BERT}},
  booktitle = {International Conference on Learning Representations (ICLR)},
  year      = {2020}
}

@article{floridi2026emperor,
  title={The Emperor's New Exploit A cautionary tale about AI agents, and about the people who run them},
  author={Floridi, Luciano},
  journal={Available at SSRN},
  year={2026}
}
}

\appendix
\section{Balanced multi-seed evaluation}
\label{app:multiseed}

For robustness we repeat the partial-trajectory evaluation under a
secondary protocol: five seeds, each redrawing a balanced
$\approx$200-trace evaluation sample (the instance partition, reference
set, and calibrated $\lambda$ are fixed across seeds).

\begin{table}[h]
\centering\small
\begin{tabular}{@{}rrllll@{}}
\toprule
$k$ & $N$ & \OtapS{} & Cosine & $\Delta$ & Seeds$+$ \\
\midrule
5  & 200 & $.618 {\pm} .047$ & $.566 {\pm} .057$ & $+.052 {\pm} .045$ & 4/5 \\
8  & 193 & $.649 {\pm} .036$ & $.581 {\pm} .032$ & $+.067 {\pm} .028$ & 5/5 \\
10 & 184 & $.649 {\pm} .020$ & $.588 {\pm} .022$ & $+.061 {\pm} .022$ & 5/5 \\
15 & 155 & $.656 {\pm} .032$ & $.629 {\pm} .007$ & $+.027 {\pm} .027$ & 4/5 \\
20 & 128 & $.667 {\pm} .010$ & $.653 {\pm} .022$ & $+.013 {\pm} .022$ & 3/5 \\
\bottomrule
\end{tabular}
\caption{Secondary (balanced multi-seed) protocol: mean $\pm$ std over
5 seeds; Seeds$+$ counts seeds with $\Delta > 0$. Absolute AUROC runs a
few points higher than the primary protocol (balanced samples over-weight
the resolved class relative to the 81\%-unresolved population) and the
deltas are somewhat larger; the direction and the early-$k$ concentration
of the advantage agree with Table~\ref{tab:partial}, whose numbers we
cite throughout.}
\label{tab:multiseed}
\end{table}

\end{document}